\documentclass[12pt]{article}
 \usepackage{amsfonts,natbib}
\usepackage{amsmath}

\usepackage{amssymb}
\usepackage{amsthm}

\usepackage{caption} % added by KP
\usepackage[table, dvipsnames]{xcolor} %added by KP
\usepackage{subcaption} %added by KP
\usepackage{hhline} %added by KP
\usepackage[bottom]{footmisc} %added by KP % Helps with footnote positioning
\usepackage{booktabs} % added by KP... for tables
\usepackage{threeparttable} %added by KP   % For footnotes in tables
\usepackage{adjustbox} % added by KP, keeps tables in page
\usepackage{enumitem} % added by KP, for compact lists
\usepackage{listings}
\usepackage{xcolor}

\definecolor{gray}{rgb}{0.5,0.5,0.5}
\definecolor{TableHighlight}{rgb}{0.99,0.92,0.87}
\lstdefinestyle{pythonstyle}{
    language=Python,
    basicstyle=\ttfamily\tiny, % Reduce font size
    commentstyle=\color{green!40!black},
    keywordstyle=\color{blue},
    numberstyle=\tiny\color{gray},
    numbers=left,
    stepnumber=1,
    numbersep=5pt,
    backgroundcolor=\color{gray!5},
    frame=single,
    breaklines=true, % Enable code wrapping
    breakatwhitespace=true,
    showstringspaces=false,
    captionpos=b,
    tabsize=4,
}
\usepackage{url}
\usepackage{listings} % added by KM

\graphicspath{{figs/}}

\usepackage{natbib}
\usepackage{color}
\newtheorem{theorem}{Theorem}
\newtheorem{lemma}[theorem]{Lemma}
\newtheorem{Lemma*}[theorem]{Lemma}
\newtheorem{remark}[theorem]{Remark}

\newtheorem{definition}[theorem]{Definition}

\newtheorem{example}[theorem]{Example}

\def\jdlqed{\vbox{\hrule \hbox{\vrule\hbox to
5pt{\vbox to 6pt{\vfil}\hfil}\vrule}\hrule}}

\newcommand{\norm}[2]{{\left\Vert #1 \right\Vert}_{#2}}

\DeclareMathOperator*{\argmin}{arg\,min}

\def\comment#1{\textit{[#1]}}
\def\comment#1{}

\begin{document}

%\begin{frontmatter}

\title{Tropical Decision Boundaries for Neural Networks Are Robust
  Against Adversarial Attacks}
\author{Kurt Pasque, Christopher Teska, Ruriko Yoshida,\\ Keiji Miura,
  and Jefferson Huang}
% How about this title? KM
% \title{Robust Neural Networks against Adversarial Attacks via Tropical Geometry}
%\author[1]{Kurt Pasque (\url{kurt.pasque@nps.edu})}
%\author[1]{Christopher Teska (\url{christopher.teska@nps.edu})}
%\author[1]{Ruriko Yoshida (\url{ryoshida@nps.edu})}
%\author[2]{Keiji Miura (\url{miura@kwansei.ac.jp})}
%\author[1]{Jefferson Huang (\url{jefferson.huang@nps.edu})}

%\date{24/07/23 - 01/09/23}
%\affiliation[1]{organization={Naval Postgraduate School}, Department of Operations Research
  %          addressline={1411 Cunningham Road}, 
 %           city={Monterey},
  %          postcode={93943}, 
   %         state={CA},
    %        country={USA}}

%\affiliation[2]{organization={Kwansei Gakuin University}, Department of Biosciences
  %          addressline={1 Gakuen Uegahara},
   %         city={Sanda},
    %        postcode={669-1330}, 
     %       state={Hyogo},
      %      country={JAPAN}}

\maketitle

\begin{abstract}
We introduce a simple, easy to implement, and computationally efficient tropical convolutional neural network architecture that is robust against adversarial attacks. We exploit the tropical nature of piece-wise linear neural networks by embedding the data in the \emph{tropical projective torus} in a single hidden layer which can be added to any model. We study the geometry of its decision boundary theoretically and show its robustness against adversarial attacks on image datasets using computational experiments.
\end{abstract}

%\begin{keyword}
%adversarial attacks \sep deep learning \sep tropical geometry   % keyword \sep keyword
%\end{keyword}

%\end{frontmatter}

\section{Introduction}
Artificial Neural Networks have demonstrated exceptional capability in the fields of computer vision, natural language processing, and genetics. However, they have similarly demonstrated a concerning vulnerability to adversarial attacks \cite{DBLP:conf/iclr/HuangPGDA17,papernot2018cleverhans}. As neural networks become prevalent in critical applications such as autonomous driving, healthcare, and cybersecurity, the development of adversarial defense methodologies will become central to the reliability and ultimate success of those efforts (for example, see \cite{madry2018towards,Kotyan,carlini2017evaluating,croce2019provable,croce2020provable} and references within).
Significant work in adversarial defense has focused on robust optimization for adversarial training (for example, see \cite{QIAN2022108889}). This approach involves a minimax game where the attacker attempts to maximize a loss function while a defender attempts to minimize this loss. This methodology degrades the performance of the network on clean inputs relative to the clean model \cite{10.5555/3454287.3455387} while increasing computational complexity.
% An alternative class of defenses uses preprocessing layers to filter out potentially adversarial inputs without classifying them. This method requires screening which inputs are allowed to be classified and is susceptible to filtering legitimate inputs with noise. Although similarly accurate, its recall is diminished. 

%\color{red}
Accumulating evidence supports the robustness of ``low-rank'' models against adversarial attacks in image recognition \cite{yang2019menet, DBLP:journals/corr/abs-2212-01957, wang2023transformed}.
There, the low-rank models were realized by different methods: matrix completion \cite{yang2019menet}, model compression \cite{DBLP:journals/corr/abs-2212-01957} or tensor SVD \cite{wang2023transformed},
suggesting that the low-rankness seems to be the universal key for adversarial robustness.
Then, it is expected that the tropical metric based on the tropical geometry \cite{MS, ETC}, that tends to favor a sparse structure in machine learning \cite{YOSHIDA202377, MY, tropicalLogit, tropicalNN}, may also be effective for adversarial robustness.
%\color{black}

Tropical geometry has been applied to describing the geometry of deep neural networks with piecewise linear activation functions, such as two of the most popular and widely used activators: rectified linear units (ReLUs) and maxout units.
In ~\cite{pmlr-v15-glorot11a} Glorot et al.~showed that neural networks with ReLUs work very well and outperform neural networks with traditional choices of activation functions using empirical studies.
Later, Zhang et al.~showed that the decision boundary of a deep neural network with ReLU activation functions is a tropical rational function with the max-plus algebra \cite{zhang2018tropical}.
Goodfellow et al.~\cite{pmlr-v28-goodfellow13} introduced maxout networks, deep feed-forward neural networks with maxout units whose activation is the maximum of arbitrarily many input neurons.
In \cite{Charisopoulos2018ATA} Charisopoulos and Maragos showed that the maxout activation function fits input data by a tropical polynomial in terms of the max-plus algebra.
Although this body of work has clearly shown that deep neural networks with ReLU and maxout activators can be understood as operations in terms of tropical geometry with the max-plus algebra, they only handled the neural networks whose input domain is in the Euclidean space.

Another strain of works utilized tropical geometry more natively and enabled to handle the deep neural networks whose input domain is in the {\em tropical projective torus}.
In 2023, Yoshida et al.~in \cite{tropicalNN} proposed {\em tropical neural networks}, which embeds the input vector in the tropical projective torus with {\rm tropical activation functions}.
A tropical neural network is a generalization of the tropical logistic regression model proposed by Aliatimis et al.~\cite{tropicalLogit} and a tropical activation function fits data with the tropical Laplace distribution centered around a tropical Fermat-Weber point within the same class.
Although these works truly employed the tropical metric of the input space,
they only used it to embed general input vectors, which is not necessarily an image, into the first hidden layer.

In this paper, we introduce a simple, easy to implement, and efficient convolution neural network (CNN) robust against adversarial attacks using tropical embedding layers. 
One idea to construct a decision boundary with a low-rank nature for image classification is to exploit tropical operations in the output layers.
Therefore, we propose a {\em tropical decision boundary} in the last layer which is a native operation over the tropical projective torus in terms of the max-plus algebra.
This approach results in well defined decision boundaries and robustness to adversarial attack with a minimal increase in computational complexity relative to standard piecewise-linear neural networks.  Especially we show that adversarial attacks developed by Carlini \& Wagner in \cite{Carlini17} may not be able to reach the optimal attack against our novel convolution neural networks with tropical embedded last layer due to the discrete nature of the geometry of its decision boundary.  

This paper is organized as follows.:
We begin with a primer on tropical geometry then develop a tropical decision boundary and a tropical convolution for deep neural networks. We provide a theoretical analysis of the geometry of the tropical decision boundaries and how they are learned. Finally, we demonstrate the robustness of the proposed tropical decision boundaries against adversarial attacks in computational experiments.

%Here we focus on tropical neural networks proposed by Yoshida et al.~\cite{tropicalNN} and introduce a {\em tropical convolutional neural network (CNNs)} which embeds the outputs of a CNN into the tropical projective torus. We will then describe the decision boundaries  
%\color{red}
We summarize our contributions as follows: %Our contributions include
\begin{itemize}
    \item We propose a tropical decision boundary and tropical CNN.
    \item We describe the decision boundary via tropical balls.
    \item We demonstrate robustness against adversarial attacks via some theoretical properties and experimental computations.
\end{itemize}
% (1) describe decision boundary via tropical balls;
% (2) define tropical CNN;
% (3) show robustness against adv attacks;
% (4) tropical FW points are optimal weights for tropical embedding layer(s);
% (5) gradient methods for tropical embedding layer(s); and (6) applications to image data.}

\section{Basics in Tropical Geometry and Tropical Bisectors}
% Applications to Neural Network with Piecewise Linear Activators}
Here we consider the tropical projective torus 
\[
\mathbb{R}^d/\mathbb{R}{\bf 1}:= \left\{x \in \mathbb{R}^d\mid x:=(x_1, x_2, \ldots , x_d) = (x_1+c, x_2+c, \ldots , x_d+c), \, \forall c \in \mathbb{R}\right\},
\]
where ${\bf 1} = (1, \ldots , 1) \in \mathbb{R}^{d}$.  
Note that for any $x = (x_1, x_2, \ldots , x_d) \in \mathbb{R}^d/\mathbb{R}{\bf 1}$,
\[
x = (x_1, x_2, \ldots , x_d) = (0, x_2 - x_1, \ldots , x_d - x_1) \in \mathbb{R}^d/\mathbb{R}{\bf 1},
\]
which means that the tropical projective torus $\mathbb{R}^d/\mathbb{R}{\bf 1}$ is isomorphic to $\mathbb{R}^{d-1}$. %We will consider the {\em max-plus algebra} over $\mathbb{R}^d/\mathbb{R}{\bf 1}$.

\begin{definition}[Tropical Arithmetic Operations]
    The tropical semiring $(\,\mathbb{R} \cup \{-\infty\},\oplus,\odot)\,$ with the max-plus algebra is a semiring with the tropical arithmetic operations of addition and multiplication defined as:
$$a \oplus b := \max\{a, b\}, ~~~~ a \odot b := a + b $$
for any $a, b \in \mathbb{R}\cup\{-\infty\}$.   Note that $-\infty$ is the identity element for the tropical addition operation $\oplus$, and $0$ is the identity element for the tropical multiplication operation $\odot$.
\end{definition}

\begin{definition}[Tropical Metric]\label{def:tropicalMetric}
    The {\em tropical metric} is defined for any $x = (x_1, \ldots , x_d), y = (y_1, \ldots , y_d) \in \mathbb{R}^d/\mathbb{R}{\bf 1}$ as
    \[
    d_{\rm tr}(x, y) = \max_{i\in \{1, \ldots , d\} }\{x_i - y_i\} - \min_{i\in \{1, \ldots , d\} }\{x_i - y_i\}.
    \]
\end{definition}
\begin{remark}
    The tropical metric $d_{\rm tr}$ is a well-defined metric over the tropical projective torus $\mathbb{R}^d/\mathbb{R}{\bf 1}$ \cite{MLKY}.
\end{remark}
\begin{definition}[Tropical Ball]\label{def:tropicalBall}
    A {\em tropical ball} $B_x(r)$ centered at $x \in \mathbb{R}^d/\mathbb{R}{\bf 1}$ with radius $r > 0$ under the tropical metric $d_{\rm tr}$ is defined as
\[
B_x(r) = \left\{y \in \mathbb{R}^d/\mathbb{R}{\bf 1}: d_{\rm tr}(x, y) \leq r\right\}.
\]
\end{definition}

\begin{remark}
    Any tropical ball $B_x(r) \subset \mathbb{R}^d/\mathbb{R}{\bf 1}$ can be viewed as a classical polytope in $\mathbb{R}^{d-1}$; see Theorem \ref{th:tropicalBall}.  
\end{remark}

\begin{definition}[Tropical Bisector]\label{def:bisector}
    Suppose $S \subset \mathbb{R}^d/\mathbb{R}{\bf 1}$ is a finite set.  Then the {\em tropical bisector} of $\mathcal{S}$ is defined as
    \[
    bis(S):=\left\{x \in \mathbb{R}^d/\mathbb{R}{\bf 1} \mid d_{\rm tr}(x, a) = d_{\rm tr}(x, b), \, \mbox{ for } a, b \in {S}\right\}.
    \]
\end{definition}
For a finite set $S \subset \mathbb{R}^d/\mathbb{R}{\bf 1}$ and any $a_1, \ldots , a_k \in S$, define 
\[
bis(F_{i_1}, \ldots ,F_{i_l})(\{a_1, \ldots , a_k \}) = bis(\{a_1, \ldots , a_k \}) \cap (a_1 + F_{1_1}) \cap \ldots \cap (a_k + F_{k_l})
\]
where each $F_{i_j}$ is a full dimensional cone in the face fan of the tropical ball with respect to the max-plus algebra, which is also a classical polytope, and $a_i + F_{i_j}$ is a polyhedron which is a full dimensional cone $F_{i_j}$ is translated so that its unique vertex is $a_i$ for $i = 1, \ldots k$.

\begin{definition}[Definition 1 in \cite{CJS}]
    Suppose $S \subset \mathbb{R}^d/\mathbb{R}{\bf 1}$ is a finite set. Then the set $S$ is in {\em weak general position} with respect to a tropical ball $B_x(r)$ if no pair of points lies in a hyperplane parallel to a facet of $B_x(r)$.

    For every subset $a_1, \ldots , a_k \in S$ and for each neighborhood $U_i$ around $a_i$, if
    \[
    bis(F_{i_1}, \ldots ,F_{i_l})(\{a_1, \ldots , a_k \}) = \emptyset \Longleftrightarrow bis(F_{i_1+, \ldots ,F_{i_l}})(\{a'_1, \ldots , a'_k \}) = \emptyset 
    \]
    for $a'_1, \ldots , a'_k$ with $a'_i \in U_i$, then we say the set $S$ is in {\em general position} with respect to a tropical ball $B_x(r)$.
\end{definition}

\begin{example}
    Let $d=3$.  Consider points $x =(0, 0, 0), \, y = (0, w, 0) \in \mathbb{R}^3/\mathbb{R}{\bf 1}$ in Figure \ref{fig:bisector_stability}.  Then  when $w=1$ and $w=0$, $x$ and $y$ are not in weak general position, which means that they are not in general position. When $w=-1$, $x$ and $y$ are in weak general position, but not in general position.
\end{example}

\section{Tropical Convolutional Neural Networks}

\subsection{Tropical Embedding Layer}

\begin{definition}[Tropical Embedding Layer]
A {\em tropical embedding layer} takes a vector $x \in \mathbb{R}^d/\mathbb{R}\mathbf{1}$ as input, and the activation of the $j$-th neuron in the embedding layer is
\begin{equation} \label{eq:embed}
z_j = \max_i(x_i + w^{(1)}_{ji}) - \min_i(x_i + w^{(1)}_{ji}) = d_{\rm tr}(-{\bf w}^{(1)}_{j}, x).
%z_j = \max_i(x_i + w^{(1)}_{ji}) - \textrm{2nd} \max_i(x_i + w^{(1)}_{ji})
\end{equation}
\end{definition}

\begin{remark}
Tropical embedding layers were originally developed in order to embed phylogenetic trees into a Euclidean space \cite{tropicalNN}.
But, as this paper shows, we found them to be effective for increasing the robustness of CNNs for image data.
% Tropical embedding layers were originally motivated by phylogenetic trees, but (as this paper shows) we found them to be effective for increasing the robustness of CNNs for image data \cite{tropicalNN}.
%A tropical convolution, we will define later, can also be regarded as a specialized case of the tropical embedding for images.
\end{remark}

\subsection{Structure of a Tropical Convolutional Neural Network}

Suppose $f^{L-1}: \mathbb{R}^{n_1\times n_2\times n_3} \to \mathbb{R}^d/\mathbb{R}^1$ is the map of a classification convolutional neural network. Then $f^{L-1}(x)$, with input data $x \in \mathbb{R}^{n_1\times n_2\times n_3}$ of $k$ classes, is the output of the network. We then embed the output of $f^{L-1}(x)$ into the tropical projective torus with (\ref{eq:embed}).
\[
    f_j^L(x) = \max_i\{f_i^{L-1}(x) + w_{ji}^L\} - \min_i\{f_{ji}^{L-1}(x)+ w_{ji}^L\} \\
    = d_{tr}(-w_j^L, x) \ j \in {1, \dots, k}
\]
We show in the following section that the weights, $W^L$, train towards {\em Fermat-Weber points} associated with the $k$ classes. See \cite{BSYM}  for details on Fermat-Weber points and using gradient descent to find them.  

Each dimension of the output $f^L(x) \in \mathbb{R}^k$ is the tropical distance between the input and Fermat-Weber points of each class. We then classify the input with a softmin function:
\[
    \max_{j \in {1, \ldots, k}} \frac{{e^{-d_{tr}(-w_j,x)}}}{\sum_{i=1}^k e^{-d_{tr}(-w_i,x)}}.
\]

\begin{remark}
    Using Theorem B in \cite{ZHOU2020787} combined with Theorem 12 in \cite{tropicalNN}, a tropical CNN is an universal approximator of any function  $f: \mathbb{R}^d \to \mathbb{R}$.
\end{remark}

\section{Definition and Analysis of Tropical Decision Boundaries}\label{sec:Boundary}

\subsection{Decision Boundary of a ReLU Neural Network}

Zhang et al.~in \cite{zhang2018tropical} explicitly described the decision boundary of feedforward neural networks with ReLU activations using tools from tropical geometry. We briefly summarize their work here. First, consider the output $\nu(x)$ of the first layer of a neural network with the ReLU activations,
\[
\nu (x) = \max \{Ax + b, t\}
\]
where $A \in \mathbb{Z}^{p\times d}$, $b \in \mathbb{R}^p$, and $t \in (\mathbb{R} \cup \{-\infty\})^p$.  Let 
\[
A = A_+ - A_-
\]
where $A_+, A_- \in \mathbb{Z}^{p\times d}_{\geq 0}$ respectively denote the positive and negative parts of $A$.  Zhang et al.~\cite{zhang2018tropical} noticed that $\nu(x)$ can be written as
\[
\nu (x) = \max \{Ax + b, t\} = \max \{A_+x + b, A_- x + t\} - A_- x,
\]
which is a {\em tropical rational function}, i.e., a difference of tropical polynomials.  It follows by induction that a neural network with $L-1$ ReLU layers can be written as a tropical rational function. Using this fact, they showed its decision boundary is contained in the tropical hypersurface (the solution set) of a tropical polynomial and computed a tight upper bound on the number of line segments in the associated piecewise-linear decision boundary. Specifically, Zhang et al.~used {\em zonotopes}, which are polytopes computed from the Minkowski sum of a set of vectors, to describe the tropical hypersurface of a tropical polynomial.

\subsection{Decision Boundary of a Tropical Neural Network}
Suppose that we have a sample $\mathcal{S} = \{(x_1, y_1), \ldots , (x_n, y_n)\}$ where $y_i \in \{1, \ldots , C\}$ and $x_i \in \mathbb{R}^d/\mathbb{R}^\mathbf{1}$ for $i = 1, \ldots, n$.   Here we consider a tropical neural network with a tropical embedding layer of $k_1 + k_2 + \ldots + k_C$ neurons, where $k_1, k_2, \ldots, k_C \in \mathbb{N} := \{1, 2, \ldots \}$.
We denote its true optimal weight matrix by $W^* \in \mathbb{R}^{(k_1+ \dots + k_c) \times d}$ where the weights $w^*_{j,1}, \ldots , w^*_{j,k_j} \in \mathbb{R}^{1 \times d}$  are rows mapping inputs to neurons $z_{j, 1}, \ldots , z_{j, k_j}$ associated with class $Y = j$.
Then note that the tropical neural network classifies the observation $x_i$ as
\[
Y_i = j \mbox{ if  } \argmin_{w^* \in W^*} d_{\rm tr}(x_i, -w^*) \in \{w^*_{j,1}, \ldots , w^*_{j,k_j}\},
\]
%where $W = \cup_{j \in \{1, \ldots , C\}}\{w^*_{j,1}, \ldots , w^*_{j,k_j}\}$.

Therefore the decision boundary of this tropical neural network is 
\[
\mathcal{B}:=
\left\{x \in \mathbb{R}^d/\mathbb{R}{\bf 1} \middle\vert
d_{\rm tr}(x, -w^*_{j, l_j}) = d_{\rm tr}(x, -w^*_{j', l_{j'}}), j \not = j', l_j \in [k_j], l_{j'} \in [k_{j'}]
\right\}
\]
where $[k_j]:=\{1, \ldots , k_j\}, \, [k_{j'}]:=\{1, \ldots , k_{j'}\}$.
% \[
% \mathcal{B}:=
% \left\{x \in \mathbb{R}^d/\mathbb{R}{\bf 1} \middle\vert
% \begin{array}{ccl}
%     c_i &=& c_j,\\
%      c_i &=& \min\{d_{\rm tr}(x, -w^*_{i,1}), \ldots , d_{\rm tr}(x, -w^*_{i,k_i})\},\\
%       c_j &=& \min\{d_{\rm tr}(x, -w^*_{j,1}), \ldots , d_{\rm tr}(x, -w^*_{j,k_j})\}\\
% \end{array}
% \right\}
% \]
% for $i, j \in \{1, \ldots , C\}$.
Let $x_0 \in \mathbb{R}^d/\mathbb{R}{\bf 1}$ be on the decision boundary.  Then there exist $q \in [k_i]$ and $q' \in [k_{j}]$ for $i, j \in \{1, \ldots , C\}$ such that
\[
r:= d_{\rm tr}(x_0, -w^*_{i, q}) = d_{\rm tr}(x_0, -w^*_{j, q'}).
\]
Then note that 
\[
x_0 \in \left(\partial B_{-w^*_{i, q}}(r)\right)\cap  \left(\partial B_{-w^*_{j, q'}}(r)\right),
\]
where $\partial B_{x}(r)$ is the boundary of the tropical ball $B_{x}(r)$.  
% \begin{remark}
%     When we have $k_1 = \ldots = k_C = 1$, $\mathcal{B}$ is the {\em tropical bisector} of $\{w^*_{1,1}, \ldots , w^*_{C, 1}\}$ defined in \cite{Criado_2021}.
% \end{remark}
\begin{theorem}[Section 3.1.1 in \cite{tropicalBall}]\label{th:tropicalBall}
    A tropical ball $B_{x}(r)$ is a polytrope, which is a tropical simplex (and hence a tropical polytope) that is also a classical polytope.
\end{theorem}
Since tropical balls $B_{-w^*_{i,q}}(r)$ and $B_{-w^*_{j,q'}}(r)$ are classical polytopes and Proposition 2.14 in~\citep{Zhang_Vol} shows explicit hyperplane representations (sets of defining inequalities for each tropical ball) of them. Some of the segments of the decision boundary are defined by one or more linear equations which define at least two of the tropical balls $B_{-w^*_{i,1}}(r), \ldots , B_{-w^*_{i,k_i}}(r)$ and $B_{-w^*_{j,1}}(r), \ldots , B_{-w^*_{j,k_j}}(r)$. Therefore we have the following theorem:
\begin{theorem}
    The decision boundary, $\mathcal{B}$, is defined by a subset of the equations defining tropical balls $B_{-w^*_{i,1}}(r), \ldots , B_{-w^*_{i,k_i}}(r)$ and $B_{-w^*_{j,1}}(r), \ldots , B_{-w^*_{j,k_j}}(r)$ for $i, j \in \{1, \ldots , C\}$.  %This means that $\mathcal{B}$ is the intersection and union of some of the equations defining tropical balls $B_{-w^*_{i,1}}(r), \ldots , B_{-w^*_{i,k_i}}(r)$ and $B_{-w^*_{j,1}}(r), \ldots , B_{-w^*_{j,k_j}}(r)$ stated in Proposition 2.14 in~\citep{Zhang_Vol}.
\end{theorem}

\begin{example}\label{eg:contour}
    We consider $\mathbb{R}^3/\mathbb{R}{\bf 1}$ and $C = 2$. Suppose we have $w^*_{1,1} = (5, -5, 0), \,  w^*_{2,1} = (-5, 5, 0)$.  Then the contour plot and heat map plot are shown in Figure \ref{fig:contourplot}. 
    In this case, from Example 3.7 in \cite{tropicalBall}, we can find the sets of inequalities to define $B_{-w^*_{1,1}}(r)$ and $B_{-w^*_{2,1}}(r)$ as follows:
    \[
    B_{-w^*_{1,1}}(r) = \left\{x \in \mathbb{R}^3/\mathbb{R}{\bf 1}\middle\vert
    \begin{array}{ccc}
        x_3 &=& 0, \\
         x_1 &\geq & 5 - r,\\
         x_1 &\leq & -5 + r,\\
         x_2 &\geq & 5 - r,\\
         x_2 &\leq & -5 + r,\\
         x_1 - x_2 &\leq & 10+r,\\
         x_1 - x_2 &\geq & 10-r
         \end{array}
  \right\},
    \]
    and
        \[
    B_{-w^*_{2,1}}(r) = \left\{x \in \mathbb{R}^3/\mathbb{R}{\bf 1}\middle\vert
    \begin{array}{ccc}
        x_3 &=& 0, \\
         x_1 &\geq & -5 - r,\\
         x_1 &\leq & 5 + r,\\
         x_2 &\geq & -5 - r,\\
         x_2 &\leq & 5 + r,\\
         x_1 - x_2 &\leq & -10+r,\\
         x_1 - x_2 &\geq & -10-r
         \end{array}
  \right\}.
    \]
    In this case, the decision boundary consists of the points where for $r = \frac{d_{\rm tr}(-w^*_{1,1}, -w^*_{2,1})}{2} = \frac{d_{\rm tr}((5, -5, 0), (-5, 5, 0))}{2} = 10$ the equation $x_1 - x_2 = 10 - r$, which defines the boundary of $B_{-w^*_{1,1}}(r)$, and the equation $x_1 - x_2 = -10 + r$, which defines the boundary of $B_{-w^*_{2,1}}(r)$, meet.  Thus, the decision boundary $\mathcal{B}$ is defined as
    \[
   \mathcal{B} =  \left\{x \in \mathbb{R}^3/\mathbb{R}{\bf 1}\middle\vert x_1=x_2, \, x_3 = 0\right\}.
    \]
    % The heat-map plot and contour plot of this example are shown in Figure \ref{fig:contourplot} below.
    \begin{figure}[h!]
        \centering
        \includegraphics[width=0.45\textwidth]{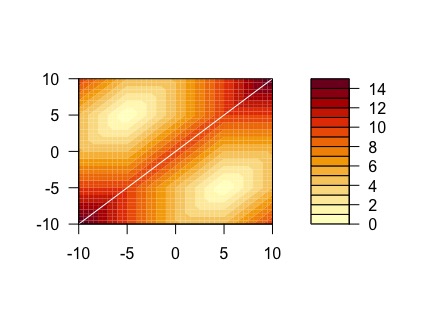}
        \includegraphics[width=0.45\textwidth]{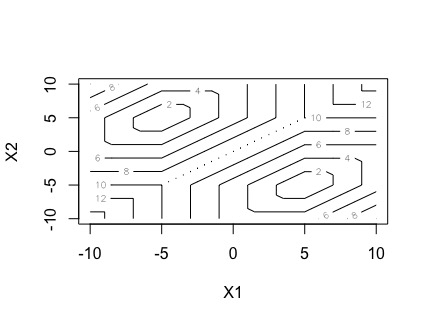}
        \caption{Here we have $w^*_{1,1} = (5, -5, 0), \,  w^*_{2,1} = (-5, 5, 0)$.  (LEFT) Heat-map plot for distance from optimal weights in Example \ref{eg:contour}.  The white line is the decision boundary. (RIGHT) Contour plot of Example \ref{eg:contour}.  As you can see tropical balls $B_{-w^*_{1,1}}(r) = B_{(5, -5)}(r)$ and $B_{-w^*_{2,1}}(r) = B_{(-5, 5)}(r)$ for $r > 0$.}
        \label{fig:contourplot}
    \end{figure}
\end{example}

\begin{example}\label{eg:contour2}
    We consider $\mathbb{R}^3/\mathbb{R}{\bf 1}$ and $C = 2$.  Suppose we have $k_1 = k_2 = 2$ and $w^*_{1,1} = (-2, -7, 0), \, w^*_{1,2} = (-8, -3, 0), \, w^*_{2,1} = (6, 1, 0), \, w^*_{2,2} = (0, 5, 0)$.  Then the contour plot and heat map plot are shown in Figure \ref{fig:contourplot2}.  Similarly to Example \ref{eg:contour}, we can compute the system of linear inequalities for each tropical ball.  Then the decision boundary for this example is a piece-wise linear line defined by 
    \[
     \mathcal{B} =  \left\{x \in \mathbb{R}^3/\mathbb{R}{\bf 1}\middle\vert 
     \begin{array}{rl}
        x_1+x_2 = 1 &\mbox{ if }x_1 \leq \frac{1}{2}, \, x_2 \geq \frac{1}{2},\\
        x_2 = \frac{1}{2}  &  \mbox{ if } \frac{1}{2}\leq x_1 \leq \frac{5}{2}, \\
        x_1+x_2 = 3 &\mbox{ if }x_1 \geq \frac{5}{2}, \, x_2 \leq \frac{1}{2}\\
     \end{array}
     \right\}.
    \]
    
    \begin{figure}[h!]
        \centering
        \includegraphics[width=0.45\textwidth]{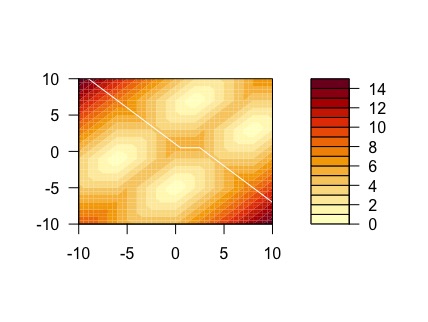}
        \includegraphics[width=0.45\textwidth]{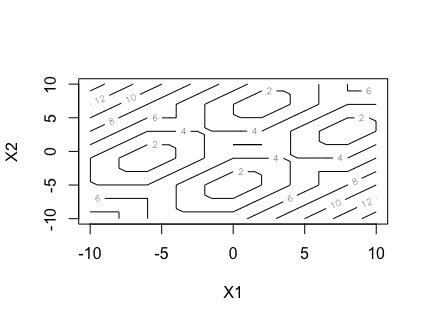}
        \caption{Here we have $w^*_{1,1} = (-2, -7, 0), \, w^*_{1,2} = (-8, -3, 0), \, w^*_{2,1} = (6, 1, 0), \, w^*_{2,2} = (0, 5, 0)$. (LEFT) Heat-map plot for distance from optimal weights in Example \ref{eg:contour}.  The white line is the decision boundary. (RIGHT) Contour plot of Example \ref{eg:contour2}.  As you can see tropical balls $B_{-w^*_{1,1}}(r) = B_{(2, 7)}(r), B_{-w^*_{1,2}}(r) = B_{(8, 3)}(r)$ and $B_{-w^*_{2,1}}(r) = B_{(-6, -1)}(r), B_{-w^*_{2,2}}(r) = B_{(0, -5)}(r)$ for $r > 0$.}
        \label{fig:contourplot2}
    \end{figure}
\end{example}

Suppose we have one neuron, i.e., $k_1 = k_2 = \ldots = k_C = 1$, for each class $c = 1, \ldots , C$ for the response variable such as Example \ref{eg:contour}.  Then we have the following lemma.
\begin{lemma}
    Let $w^*_{c, 1}$ be the optimal weight or a kernel for the neuron $z_c$ in the tropical embedding layer for $c = 1, \ldots , C$, and let $S := \{w^*_{1, 1}, \ldots , w^*_{C, 1} \}$.  Then the decision boundary of the tropical neural network is the tropical bisector $bis(S)$ defined in Definition \ref{def:bisector}. 
\end{lemma}
\begin{proof}
    This is trivial from Definition \ref{def:bisector}.
\end{proof}

% \begin{remark}
%     In the case of $C = 2$ and $k_1 = k_2 = 1$, Criado et al.~discussed details on and classified tropical bisectors of two points in \cite[Section 4]{CJS}.  
% \end{remark}

\begin{theorem}[Proposition 4 in \cite{CJS}]\label{tm:twoneurons}
    Suppose we have $C \geq 2$, and $k_1 = \ldots = k_C = 1$.  
    If the optimal weights $w^*_{t, 1}:=(w^{*1}_{t, 1}, \ldots , w^{*d}_{t, 1}), \, w^*_{t', 1}:=(w^{*1}_{t', 1}, \ldots , w^{*d}_{t', 1}) \in \mathbb{R}^d/\mathbb{R}{\bf 1}$ for $t, t' \in \{1, \ldots , C\}$ and for each neuron (or kernel) in the tropical embedding layer are in weak general position, then the decision boundary of a tropical neural network is defined by the homogeneous max-tropical Laurent polynomial such that
    \begin{equation}\label{eq:maxlauren}
        \max \left(\max_{i, j \in \{1, \ldots , d\}} (x_i - w^{*i}_{t, 1} - x_j + w^{*j}_{t, 1} ), \max_{k, l \in \{1, \ldots , d\}} (x_k - w^{*k}_{t', 1} - x_l + w^{*l}_{t', 1} ) \right).
    \end{equation}
    In addition, the decision boundary is contained in a max-tropical hypersurface of degree $d$.
\end{theorem}

Note that if $k_1=k_2 = 1$ and if $w^*_{1,1}$ and $w^*_{2,1}$ are in weak general position, then we can apply Theorem \ref{tm:twoneurons} to compute the decision boundary of a tropical neural network.  
First, we enumerate the maximal cells of the tropical hypersurface defined by the equation in \eqref{eq:maxlauren} where one of 
\[
x_i - w^{*i}_{1, 1} - x_j + w^{*j}_{1, 1}  \mbox{ for }i, j \in \{1, \ldots , d\}
\]
and one of 
\[
x_k - w^{*k}_{2, 1} - x_j + w^{*l}_{2, 1} \mbox{ for }k, l \in \{1, \ldots , d\}
\]
attain maxima. The time complexity of this algorithm is $\Omega(d^4)$ which is tight by Corollary 8 in \cite{CJS}.

\begin{theorem}\label{tm:decDim}
    Suppose we have $C = 2$ and $k_1, k_2 \geq 1$.  Then the decision boundary of the tropical neural network defined by $w^*_{1, 1}, \ldots w^*_{1, k_1}$ and $w^*_{2, 1}, \ldots w^*_{2, k_2}$ does not contain full-dimensional cells if and only if any pair of $w^*_{1, i}$ and $w^*_{2, j}$ for $i \in \{1, \ldots , k_1\}$ and $j \in \{1, \ldots , k_2\}$ is in weak general position.
\end{theorem}

\begin{proof}
    This is trivial from Proposition 2 in \cite{CJS}.
\end{proof}

Similarly we have the following theorem for $C \geq 2$ and $k_1 = \ldots = k_C = 1$.
\begin{theorem}\label{tm:decDim2}
    Suppose we have $C \geq 2$ and $k_1= \ldots = k_C = 1$.  Then the decision boundary of the tropical neural network defined by $w^*_{1, 1}, \ldots w^*_{C, 1}$ does not contain full-dimensional cells if and only if any pair of $w^*_{i, 1}$ and $w^*_{j, 1}$ for $i, j \in \{1, \ldots , C\}$ is in weak general position.
\end{theorem}

\begin{proof}
    This is trivial from Proposition 2 in \cite{CJS}.
\end{proof}

% \begin{remark}
%     See appendix for further examples on general position.
% \end{remark}
\begin{example}
   Consider $w^*_{1,1}=(5, -5, 0)$ and $w^*_{2, 1} = (-5, 5, 0)$ from Example \ref{eg:contour2}.  Here $w^*_{1, 1}$ and $w^*_{2, 1}$ are in weak general position.  Thus, by Theorem \ref{tm:decDim}, the decision boundary has  dimension less than $d-1 = 2$.  In this case it has dimension $1$.
\end{example}

\begin{example}\label{eg:contourNonFull}
   Suppose we have $C = 2$ and $k_1= k_2 = 1$.  Consider $w^*_{1,1}=(5, 5, 0)$ and $w^*_{2, 1} = (-7, -7, 0)$ which are not in weak general position since they are on the hyperplane $x_1 = x_2$.  Figure \ref{fig:contourplotNonFull} shows contour plots of two tropical Laplacian distributions.  In this case $w^*_{1,1}$ and $w^*_{2,1}$ are not in weak general position and the decision boundary has full dimensional, i.e., dimension $2$, region(s) by Theorem \ref{tm:decDim}.  In this case we have the region
   \[
   \mathcal{B} =  \left\{x \in \mathbb{R}^3/\mathbb{R}{\bf 1}\middle\vert 
   \begin{array}{rcl}
        x_1 & \leq & -7   \\
       x_2 & \geq & 5\\
        x_1 + x_2 &=& -2\\
        x_1 & \geq & 5   \\
        x_2 & \leq & -7\\
   \end{array}
   \right\} .
   \]
       \begin{figure}[h!]
       \centering
       \includegraphics[width=0.45\textwidth]{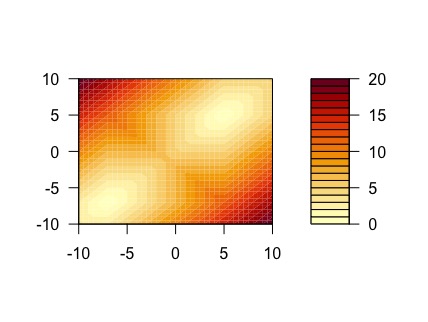}
        \includegraphics[width=0.45\textwidth]{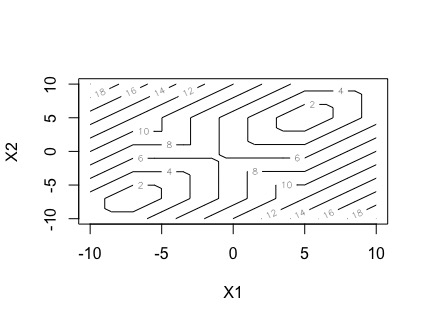}
       \caption{Here we have $w^*_{1,1} = (5, 5, 0), \, w^*_{2,1} = (-7, -7, 0)$. (LEFT) Heat-map plot for distance from optimal weights in Example \ref{eg:contourNonFull}.  (RIGHT) Contour plot of Example \ref{eg:contourNonFull}.  }
        \label{fig:contourplotNonFull}
   \end{figure}
\end{example}

It is possible to classify all the possible configurations of two points in the planar case or in $\mathbb{R}^3/\mathbb{R}{\bf 1}$.
By scaling and parallel translation, you can set the coordinates of the two points to $w^*_{1, 1}=(0, 0, 0)$ and $w^*_{2, 1} = (1, w, 0)$ without loss of generality. The configurations are classified exhaustively as in the following lemma with Figure \ref{fig:bisector_stability}.

\begin{figure}[h!]
    \centering
    \includegraphics[width=1\textwidth]{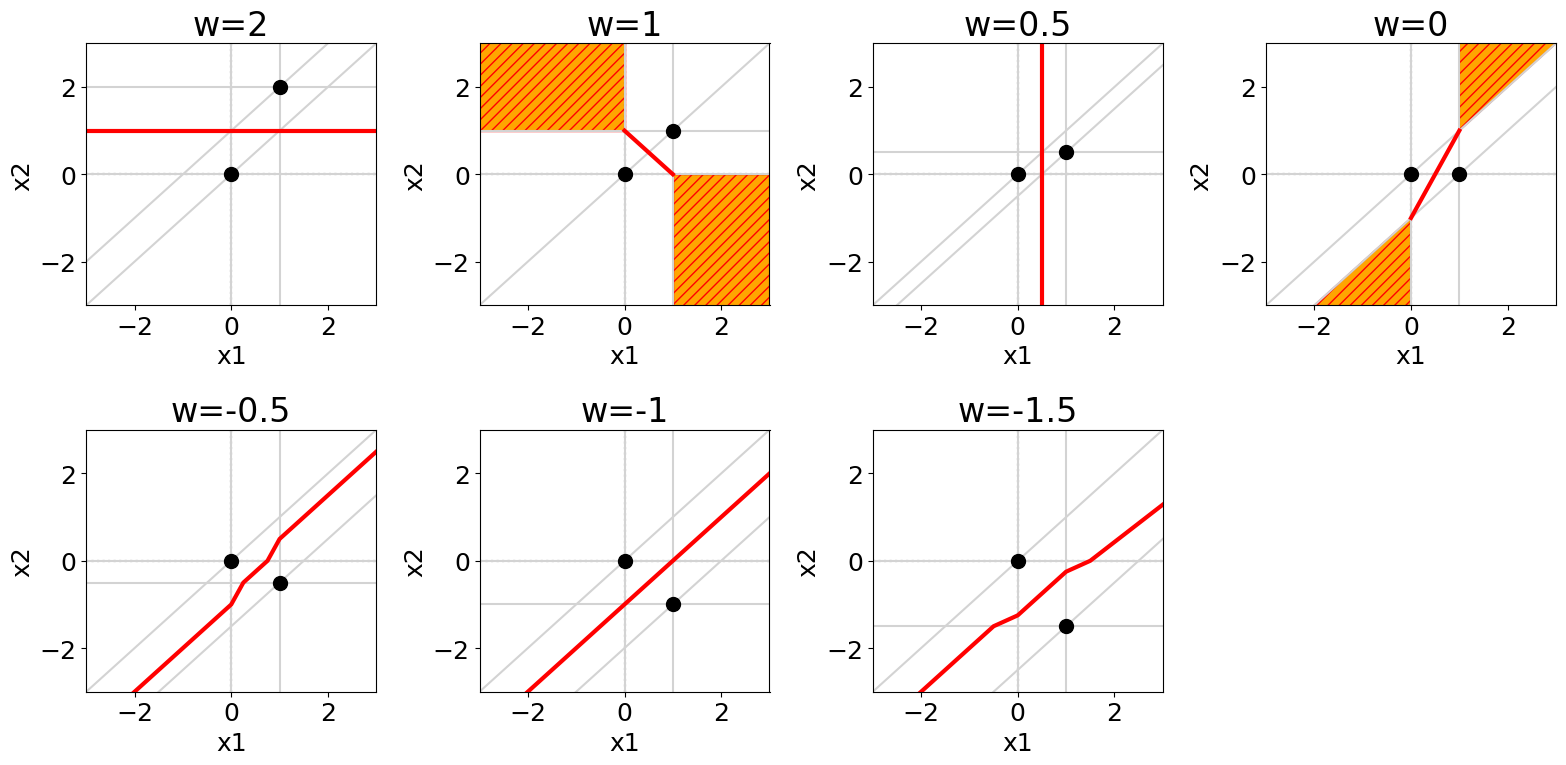}
    \caption{The bisectors between $w^*_{1, 1}=(0, 0, 0)$ and $w^*_{2, 1} = (1, w, 0)$ for various $w$ are represented by red. The lightgray lines represent the hyperplanes for $w^*_{1, 1}$ and $w^*_{2, 1}$. Note that $w=1$ and $w=0$ are not in weakly general positions (i. e. $w^*_{2, 1} - w^*_{1, 1}$ is parallel to a facet of a tropical unit ball) and, therefore, not in general positions (i. e. small perturbations change which sectors the bisector is in). $w=-1$ is in weakly general positions but not in general positions.}
    \label{fig:bisector_stability}
\end{figure}

\begin{lemma}
\label{lemma:bisector_classification}
Suppose we have a binary response variable, i.e., $C = 2$, and $k_1 = k_2 = 1$.
Let $w^*_{1, 1}=(0, 0, 0)$ and $w^*_{2, 1} = (1, w, 0)$ in $\mathbb{R}^3/\mathbb{R}{\bf 1}$.
Then, for $1 < w$, the bisector representing the decision boundary for the two points is $x_2=w/2$.
For $0 < w < 1$, it is $x_1=1/2$.
For $-1 < w < 0$, it is
\[
x_2 =
\left\{
\begin{array}{ll}
x_1 -1 & (x_1 < 0) \\
2 x_1 -1 & (0< x_1 < 1/2+w/2) \\
x_1 - 1/2 + w/2 & (1/2+w/2 < x_1 < 1/2-w/2) \\
2 x_1 -1+w & (1/2-w/2 < x_1 <1) \\
x_1 + w & (1 < x_1)
\end{array}
\right.
\]
For $w = -1$, it is $x_2 = x_1 - 1$.
(Note that it is low dimensional and this straight line can also be obtained by setting $w=-1$ in the above and below equations for $-1 < w < 0$ and $w < -1$.
This contrasts with the full-dimensional bisectors for $w=1$ and $w=0$ as in Figure \ref{fig:bisector_stability}.)
For $w < -1$, it is
\[
x_2 =
\left\{
\begin{array}{ll}
x_1 -1 & (x_1 < -1-w) \\
1/2 x_1 + w/2-1/2 & (-1-w < x_1 < 0) \\
x_1 + w/2 - 1/2 & (0 < x_1 < 1) \\
1/2 x_1 + w/2 & (1 < x_1 < -w) \\
x_1 + w & (-w < x_1)
\end{array}
\right.
\]
\end{lemma}
\begin{proof}
Direct calculations.
\end{proof}

% \begin{theorem}
% When $w=1$ or $w=0$, the two points in Lemma \ref{lemma:bisector_classification} are not in weakly general positions and, therefore, not in general positions.
% When $w=-1$, they are in a weakly general position but not in a general position.
% For the other cases, they are in a general position. 
% \end{theorem}
% \begin{proof}
% When $w \neq 1, 0, -1$, the bisector is a continuous function on $w$ as described in Lemma \ref{lemma:bisector_classification}.
% Thus perturbing $w$ never changes which sectors of the two hyperplanes for $w^*_{1, 1}=(0, 0, 0)$ and $w^*_{2, 1}=(1, w, 0)$ in $\mathbb{R}^3/\mathbb{R}{\bf 1}$ the bisector is in.

% For $w=1$, $w^*_{2, 1} - w^*_{1, 1} = (1, 1, 0)$, which is parallel to a facet of a unit ball.
% Similarly, for $w=0$, $w^*_{2, 1} - w^*_{1, 1} = (1, 0, 0)$, which is parallel to a facet of a unit ball.
% Not being in a weakly general position implies not being in a general position.

% Across $w=-1$, which sectors the bisector is in changes discontinuously.
% \end{proof}

In the above example, $w=1$ and $w=0$ are not in weakly general positions (i. e. $w^*_{2, 1} - w^*_{1, 1}$ is parallel to a facet of a unit ball) and, therefore, not in general positions (i. e. small perturbations change which sectors the bisector is in). $w=-1$ is in weakly general position but not in general position.
As shown before, not being in a weakly general position  results in full-dimensional bisectors. In fact, perturbing $w$ from $1$ or $0$ drastically changes the positions of bisectors or which sectors the bisector exists.

Intuitively, this drastic change is required to connect the horizontal bisector for $w>1$ and the vertical bisector for $0<w<1$ via the full-dimensional bisector for $w=1$.
Similarly, the vertical bisector for $0<w<1$ and the diagonal bisector for $-1<w<0$ are connected via the ful-dimensional bisector for $w=0$. Thus, the structural stability follows from being or not being in general positions.
Not being in a general position, by definition, leads to the sensitivity of the bisector positions to small perturbations.

Interestingly, $w=-1$ is in a weakly general position but not in a general position. Perturbing $w$ from $-1$ somehow changes the positions of bisectors or which sectors the bisector is in, but the vulnerability is rather mild. The positions of the bisector do switch but the overall shape of the bisector is more or less similar.
In fact, it consists of five line segments when $w$ is close to $-1$ while it is just a single straight line when $w=-1$.

In the example shown in Figure~\ref{fig:bisector_stability} and Lemma~\ref{lemma:bisector_classification}, it may be illuminating to illustrate with some simplification how the sector boundary is modified during transfer learning from the viewpoint of Fisher information.

\begin{lemma}
\label{lemma:fisher_info}
Suppose that the three-dimensional feature $x$ is embedded into the last layer as $(d_{\rm tr}(x, w_1), d_{\rm tr}(x, w_2))$, where we only learn a single parameter $w$ in the weights $w_1=(0, 0, 0)$ and $w_2=(1, w, 0)$, whose true value is $w^*=2$.
Furthermore, suppose that the output of the neural network, that represents the probability for $y=1$ as a softmax, is given by $\hat{y} = \frac{1}{1 + e^{d_{\rm tr}(x, w_2) - d_{\rm tr}(x, w_1)}}$ and there are only two possibilities for explaining variables:
$x=(0, -1, 0)$ for $y \simeq 0$ and $x=(1, 3, 0)$ for $y \simeq 1$.
% $x=(0, 0, 0)$ for $y \simeq 0$ and $x=(1, 2, 0)$ for $y \simeq 1$.
Then the Fisher information for the log-likelihood function (= negative loss),
$l = y \log \hat{y} + (1-y)\log (1-\hat{y})$, is $I=0.1$.
That is, $\textrm{Var}[\hat{w}_{MLE}] \simeq \frac{1}{n I} =\frac{10}{n}$ where $n$ is the number of observations.
\end{lemma}
\begin{proof}
By direct calculation under $w^*=2$, we have
$\Delta := d_{\rm tr}(x, w_2) - d_{\rm tr}(x, w_1) 
= w \simeq 2$ for $x=(0, -1, 0)$ and $\Delta = -w \simeq -2$ for $x=(1, 3, 0)$.
%= w \simeq 2$ for $x=(0, 0, 0)$ and $\Delta = |w-2|-2 \simeq -2$ for $x=(1, 2, 0)$.
Thus, $\frac{\partial \Delta}{\partial w} = \pm 1$.
Then,
$I := \textrm{E}[(\frac{\partial l}{\partial \hat{y}} \frac{\partial \hat{y}}{\partial \Delta} \frac{\partial \Delta}{\partial w})^2]
= \textrm{E}[\{ \frac{y}{\hat{y}} - \frac{1-y}{1-\hat{y}} \}^2 \{\hat{y}(1-\hat{y})\}^2 ]
= \textrm{E}[(y-\hat{y})^2]
= p(y=1)(1-\hat{y})^2 + p(y=0) (-\hat{y})^2
= \hat{y}(1-\hat{y})^2 + (1-\hat{y}) (-\hat{y})^2
= \hat{y}(1-\hat{y}) = 0.1$.
The final numerical value is common for $x=(0,-1,0)$ and $x = (1,3,0)$ due to their symmetry in position.
\end{proof}

\begin{remark}
In this example, we considered a two-point distribution for $x$.
However, the value of Fisher information does not necessarily strongly depend on the distribution of $x$.
For example, we can consider a single-point distribution (= delta function) in which $x$ is always $(0.5, 1, 0)$.
As $\hat{y}=0.5$ there, we have $I=0.25$, which is obviously the maximum of $I = \hat{y}(1-\hat{y})$.
\end{remark}

In general if we have all observations in the sample $\mathcal{S} \subset \mathbb{R}^d/\mathbb{R}{\bf 1}$ are all in the weak general position, then we have the following theorem from \cite{Criado_2021}.  
\begin{definition}
    Suppose we have a set $S \subset \mathbb{R}^d/\mathbb{R}{\bf 1} \cong \mathbb{R}^{d-1}$. $S$ is star convex with center $x_0 \in \mathbb{R}^d/\mathbb{R}{\bf 1}$ if for any point $x \in S$
the ordinary line segment $[x_0, x]$ is contained in $S$. 
\end{definition}

\begin{theorem}[Theorem 6 in \cite{Criado_2021}]\label{tm:polytrope}
    If all observations in the sample $\mathcal{S} \subset \mathbb{R}^d/\mathbb{R}{\bf 1}$ are in weak general position, then  each tropical Voronoi region of $\mathcal{S}$ is the star convex union of finitely many (possibly unbounded)
semi-polytropes.
\end{theorem}

    Many researchers study the geometry of a polytrope, which is both classically convex and tropically convex (for example, \cite{TRAN20171,joswigBook,Tran}).  It is well-known that a polytrope is a tropical simplex in $\mathbb{R}^d/\mathbb{R}{\bf 1}$ and their hyperplane representations on polytropes \cite{TRAN20171}.  Tran in \cite{TRAN20171} showed the hyperplane-representation of a polytrope $\mathcal{P}$ can be constructed from an associated \textit{Kleene Star} weight $d\times d$ matrix, $\mathbf{m^*}$ such that
    \begin{equation}\label{eq:kleen}
    \mathcal{P}=\{\mathbf{y}\in\mathbb{R}^d\;|\; y_j-y_i\leq -m_{ij}, y_1=0, m_{ij}\in \mathbf{m^*}, i \neq j\},
\end{equation}
where $m_{ij}$ is the $(i,j)$-th entry in $\mathbf{m^*}$. 
\begin{theorem}
    The decision boundary of the tropical CNN is union of hyperplanes defined by inequalities in \eqref{eq:kleen}.
\end{theorem}

\begin{definition}[Tropical Fermat-Weber point]
A {\em tropical Fermat-Weber point} $x^*$ of a sample $\mathcal{S} = \{p_1, \ldots p_n\} \subset \mathbb{R}^d/\mathbb{R}{\bf 1}$ with respect to the {tropical metric} $d_{\rm tr}$ over the tropical projective torus $\mathbb{R}^d/\mathbb{R}{\bf 1}$ is defined by
\begin{equation}\label{eq:trop_FW}
x^* = \argmin_x \sum_{i=1}^n {d}_{\rm tr}(x, p_i).
\end{equation}    
\end{definition}

Let $w_{c} \in \mathbb{R}^d/\mathbb{R}{\bf 1}$ be a tropical Fermat-Weber point for a class $c \in [C]:=\{1, \ldots , C\}$ and we assume that each observation $X:=(X_1, \ldots , X_d)$ for each class $c \in [C]$ is distributed according to the Gaussian distribution around $w_{c}$ with the covariant matrix $\sigma I_d$, where $\sigma > 0$ and $I_d$ is the $d\times d$ identify matrix.  
% Then the vertices of the tropical ball $B_{w^{c}}(l)$ for $l > 0$ is 
% \[V=\left\{\begin{array}{ccc}
%     (w_c^1-l,w_c^2-l,w_c^3-l\ldots,w_c^d-l),\\
%      (w_c^1+l,w_c^2+l,w_c^3+l\ldots,w_c^d+l),\\
%     (w_c^1+l,w_c^2,w_c^3,\ldots,w_c^d),\\
%     (w_c^1-l,w_c^2,w_c^3,\ldots,w_c^d),\\
%     \vdots\\
%     (w_c^1,w_c^2,w_c^3,\ldots,w_c^d+l),\\
%     (w_c^1,w_c^2,w_c^3,\ldots,w_c^d-l).
% \end{array}\right\}.\]
% The expected distance from the decision boundary which, in this case, is the decision boundary from a tropical hyperplane $H_{-v_i}$ with the normal vector $v_i \in V$ for some $i \in [d]$ which is on the decision boundary.
% \begin{lemma}\label{lm:infNorm}
%     Suppose $x = (x_1, \ldots , x_d)\in \mathbb{R}^d/\mathbb{R}{\bf 1}$.  Then we have
%     \[
%     d_{\rm tr}(0, x) \leq 2 \max_i |x_i| = 2 ||x||_{\infty}.
%     \]
% \end{lemma}
% \begin{proof}
%     \begin{eqnarray*}
%     d_{\rm tr}(0, x) &=& \max x_i - \min x_i\\
%     &=& \max x_i + \max -x_i\\
%     &\leq & 2 \max|x_i|\\
%     &=& 2||x||_{\infty}.
%     \end{eqnarray*}
% \end{proof}

Therefore we have the following theorems:
\begin{theorem}
    The distribution of the tropical distance from $d_{\rm tr}(w_{c}, X)$ is
    \[
    F(t) = P(d_{\rm tr}(w_{c}, X) \leq t) = d \int_{-\infty}^{\infty}[G(\sigma t + x) - G(x)]^{d-1}G'(x)dx
    \]
    for $t > 0$ where
    \[
    G(x) = \int_{-\infty}^x G'(t)dt, \mbox{ and } G'(t) = \frac{\exp(-\frac{t^2}{2})}{\sqrt{2\pi}}.
    \]
\end{theorem}
\begin{proof}
    Without loss of generality we set $w_{c}:=(w_c^1, \ldots , w_c^d) = (0, \ldots , 0)$. Then we have 
    \[
    d_{\rm tr}(0, X) = \max X_i - \min X_i .
    \]
    This is the distribution of the range of $d$ i.i.d. normal random variables with mean $0$ and standard deviation  $\sigma > 0$.  Then we use the result from \cite{range}.
\end{proof}
\begin{theorem}\label{upperbound}
Let \[
    E(l, d) = \frac{\sigma d!}{(l - 1)!(d - l)!}\int_{-\infty}^{\infty}x (1 - \Phi (x))^{l-1} (\Phi)^{d-l}\phi(x) dx.
    \]
Let $w^{(c)}$ be a tropical Fermat-Weber point.
Then we have
    \begin{eqnarray*}
        P(d_{\rm tr}(w_{c}, X) \geq r) &\leq &\frac{E(d, d) - E(1, d)}{r}
    \end{eqnarray*}
    for some distance $r > 0$.
\end{theorem}
%\begin{proof} See Appendix \end{proof}
\begin{proof}
   Without loss of generality we set $w_{c}:=(w_c^1, \ldots , w_c^d) = (0, \ldots , 0)$.  By \cite{OrderedStat}, we have the expectation of $l$th ordered statistic of $d$ many i.i.d. standard normal random variables is
   \[
   E'(l, d) = \frac{d!}{(l - 1)!(d - l)!}\int_{-\infty}^{\infty}x (1 - \Phi (x))^{l-1} (\Phi)^{d-l}\phi(x) dx.
   \]
   So the expectation of $l$th ordered statistic of $d$ many i.i.d.  normal random variables around $0$ with its standard deviation $\sigma$ is
   \[
   E(l, d) = \frac{\sigma d!}{(l - 1)!(d - l)!}\int_{-\infty}^{\infty}x (1 - \Phi (x))^{l-1} (\Phi)^{d-l}\phi(x) dx.
   \]
   Thus we have 
   \[
   d_{\rm tr}(0, X) = \max X_i - \min X_i = E(d, d) - E(1, d).
   \]
   Then by Markov inequality we have
   \begin{eqnarray*}
       P(d_{\rm tr}(0, X) \geq r) &\leq &\frac{E(d, d) - E(1, d)}{r}
   \end{eqnarray*}
   for $r > 0$.  
\end{proof}

\begin{remark}
    Barnhill et al.~in \cite{BSYM} showed that a tropical Fermat-Weber point is very stable by Theorem 3 and it is robust against outlier(s).  
\end{remark}

\section{Computational Experiments}\label{sec:comp_experiments} %\section{Attacking and Defending Classifiers}

Our experiment is to show that using a tropical layer as the final layer of a standard, convolutional neural network can result in a model that is as accurate as baseline models, requires no additional parameters, trains in a similar time, and retains more prediction power when predicting adversarially perturbed input data when compared to baseline models. Our experiment was completed in Python using Tensorflow and high performance computing resources.%\footnote{All code utilized in experiments can be found in our GitHub repository: https://github.com/KurtPask/TropicalNN.}. 
Our method begins by training tropical CNN models and other benchmark models to predict labels on the training sets of three benchmark image datasets: MNIST~\cite{MNIST}, SVHN~\cite{SVHN}, and CIFAR-10~\cite{CIFAR10}. Following training, we attacked the test set of data using well-known techniques with varying norm constraints. To evaluate our model's robust characteristics, we compare our tropical CNN's test set error percentage against the test set error percentage of other baseline models. To ensure comparisons are appropriately made, models for each dataset have the same base model, and only differ in the final layer they use or in the regularization technique employed. Benchmark models are both clean and adversarially trained~\cite{madry2019deep}~\cite{goodfellow2015explaining} ReLU and Maxout~\cite{goodfellow2013maxout} models and the Maximum Margin Regularizer - Universal (MMR) adapted from ~\cite{croce2020provable}. Further details on model construction, tropical layer implementation, attack hyperparameters, and computation times are in Appendix \ref{appendix:experiment}. 

For the CIFAR-10 dataset, we use a ResNet50 model~\cite{he2015deep} as our base model\footnote{MMR was not evaluated for CIFAR-10 as re-creating the MMR-Universal regularizer for a model as large as ResNet50 exceeded our computational budget for the research}. For MNIST and SVHN, we use a more simple convolutional neural network with three convolution layers and one fully connected layer. More details on the simple base model is outlined in Appendix Table \ref{table:base-model}.

To evaluate robustness, we utilized the following attacks
%We attacked our models with 5 common attacks across 3 norms: $\ell_1$, $\ell_2$, and $\ell_\infty$. Background on each of these attacks is well defined in the Appendix Section \ref{section:attacks}. The attacks utilized and the acronyms we refer to them are specified below and further details on the hyperparameters used in each attack are in Appendix \ref{}

\begin{itemize}[itemsep=0pt,parsep=0pt,left=0pt]
    \item \emph{SLIDE}. $\ell_1$ %attack %implements the Sparse $\ell_1$ Descent (SLIDE) 
    attack defined in~\cite{tramèr2019adversarial}. 
    \item \emph{PGD}. $\ell_2$ and $\ell_\infty$ Projected Gradient Descent from ~\cite{madry2019deep}. %We implement projected gradient descent (PGD) using both $\ell_2$ and $\ell_\infty$ norm constraints. 
    For $\ell_2$, we use common Tensorflow methods to implement. For $\ell_\infty$, we use the implementation in the CleverHans GitHub repository~\cite{papernot2018cleverhans}.
    %\item $FGSM$. We implement an $\ell_\infty$ fast gradient sign method (FGSM) from the CleverHans repository.
    \item \emph{CW}. The $\ell_2$ Carlini and Wagner attack defined in~\cite{carlini2017evaluating} using the implementation from the CleverHans repository. 
    \item \emph{SPSA}. %We implemented the simultaneous perturbation stochastic approximation (SPSA) algorithm found in the 
    Gradient-free $\ell_\infty$ attack utilizing CleverHans implementation of SPSA from~\cite{uesato2018adversarial}.
\end{itemize}

\subsection{Results}\label{subsec:result}

The results from our experiment are in Table \ref{tab:mnist_attack_results}, \ref{tab:svhn_attack_results}, and \ref{tab:cifar_attack_results} for our models trained on the MNIST, SVHN, and CIFAR-10 datasets, respectively.
In MNIST and CIFAR in particular, our tropical CNN outperformed the ReLU and maxout models against all attacks when trained normally. More is said on the adversarially trained models in the Discussion.
Another standout result from each of these tables is that the CW attack, a powerful attack, routinely performed worse on the tropical CNN, compared to the other models. We describe some possibilities as to why it is unable to find an adversarial example in Appendix \ref{section:attacks}. 

\begin{table}[h]
\centering
\begin{adjustbox}{width=\columnwidth}
\begin{threeparttable}
\caption{MNIST results. Values reported are error percentage on the test set.}
\label{tab:mnist_attack_results}
\begin{tabular}{lcccccc}\toprule%{l|c|c|c|c|c|c}
 & & \textbf{$\ell_1$ ($\epsilon=5.6$)} & \multicolumn{2}{c}{\textbf{$\ell_2$ ($\epsilon=2.8$)}} & \multicolumn{2}{c}{\textbf{$\ell_\infty$ ($\epsilon=0.1$)}}  \\
 \cmidrule(lr){3-3}\cmidrule(lr){4-5}\cmidrule(lr){6-7}
 \textbf{Model} & \textbf{Clean}& \textbf{SLIDE} & \textbf{PGD} & \textbf{CW (mean $\ell_2$)\tnote{*}}  & \textbf{PGD} & \textbf{SPSA} \\
\midrule
ReLU & 0.99 \% & 20.93 \% & 36.39 \% & 99.25 \% (2.28)  & 16.44 \% & 30.33 \%\\
Maxout & 0.81 & 56.74 & 92.37 & 99.41 (2.22) & 39.78 & 29.45 \\
\rowcolor{TableHighlight}
Tropical & 0.71 & 15.01 & 27.59 & 5.97 (1.48)  & 8.74 & 3.91 \\
\midrule
ReLU+AT\tnote{\dag} & 0.61 & 9.81 & 15.34 & 99.55 (3.89) & 3.19 & 2.94 \\
Maxout+AT & 0.87 & 22.69 & 28.08 & 99.26 (3.71) & 3.91 & 3.51 \\
\rowcolor{TableHighlight}
Tropical+AT & 0.66 & 11.15 & 12.37 & 22.38 (3.25) & 3.21 & 2.77 \\
MMR & 0.66 & 0.69 & 0.74 & 99.46 (5.59) & 0.78 & 12.54 \\
\bottomrule
\end{tabular}
    \begin{tablenotes}
        \item[*] The mean $\ell_2$ distortion for the test set for adversarial examples found. If test error for $CW<100\%$, then mean computed only on adversarial examples the CW algorithm was able to find. Higher distortion indicates it is more difficult to find adversarial examples for the model.
        \item[\dag] +AT indicates the model was trained with examples that had been perturbed using the $\ell_\infty$ PGD attack as described in Section \ref{section:AT}
    \end{tablenotes}
\end{threeparttable}
\end{adjustbox}
\end{table}

\begin{table}[h]
\centering
\begin{adjustbox}{width=\columnwidth}
\begin{threeparttable}
\caption{SVHN results. Values reported are error percentage on the test set.}
\label{tab:svhn_attack_results}
\begin{tabular}{lcccccc}\toprule%{l|c|c|c|c|c|c}
 & & \textbf{$\ell_1$ ($\epsilon=1.74$)} & \multicolumn{2}{c}{\textbf{$\ell_2$ ($\epsilon=0.87$)}} & \multicolumn{2}{c}{\textbf{$\ell_\infty$ ($\epsilon=\frac{4}{255}$)}}  \\
 \cmidrule(lr){3-3}\cmidrule(lr){4-5}\cmidrule(lr){6-7}
 \textbf{Model} & \textbf{Clean}& \textbf{SLIDE} & \textbf{PGD} & \textbf{CW (mean $\ell_2$)} & \textbf{PGD} & \textbf{SPSA} \\
\midrule
ReLU & 9.96 \% & 43.42 \% & 66.85 \% & 94.92 \% (0.63) & 61.88 \% & 85.67 \% \\
Maxout & 10.54 & 67.01 & 96.52 & 94.38 (0.47) & 95.88 & 92.51 \\
\rowcolor{TableHighlight}
Tropical & 10.56 & 42.04 & 75.17 & 42.20 (0.82) & 68.65 & 36.33 \\
\midrule
ReLU+AT & 9.41 & 29.76 & 42.77 & 95.05 (0.97) & 35.74 & 28.60 \\
Maxout+AT & 8.66 & 29.50 & 43.14 & 95.31 (0.92) & 35.25 & 30.21 \\
\rowcolor{TableHighlight}
Tropical+AT & 11.09 & 32.84 & 44.02 & 93.75 (0.94) & 37.20 & 31.04 \\
MMR & 12.69 & 27.81 & 31.82 & 93.76 (1.02) & 31.46 & 75.51 \\
\bottomrule
\end{tabular}
\end{threeparttable}
\end{adjustbox}
\end{table}

\begin{table}[h]
\centering
\begin{adjustbox}{width=\columnwidth}
\begin{threeparttable}
\caption{CIFAR-10 results. Values reported are error percentage on the test set.}
\label{tab:cifar_attack_results}
\begin{tabular}{lcccccc}\toprule%{l|c|c|c|c|c|c}
 & & \textbf{$\ell_1$ ($\epsilon=1.74$)} & \multicolumn{2}{c}{\textbf{$\ell_2$ ($\epsilon=0.87$)}} & \multicolumn{2}{c}{\textbf{$\ell_\infty$ ($\epsilon=\frac{4}{255}$)}}  \\
 \cmidrule(lr){3-3}\cmidrule(lr){4-5}\cmidrule(lr){6-7}
 \textbf{Model} & \textbf{Clean}& \textbf{SLIDE} & \textbf{PGD} & \textbf{CW (mean $\ell_2$)} & \textbf{PGD} & \textbf{SPSA} \\
\midrule
ReLU & 29.61 \% & 54.51 \% & 88.60 \% & 85.48 \% (0.55) & 86.72 \% & 83.79 \% \\
Maxout & 27.74 & 54.48 & 97.92 & 86.32 (0.53) & 95.31 & 85.53 \\
\rowcolor{TableHighlight}
Tropical & 29.13 & 51.81 & 73.39 & 60.1 (0.87) & 71.54 & 59.86 \\
\midrule
ReLU+AT & 32.01 & 42.30 & 67.46 & 85.31 (1.15) & 65.08 & 55.16 \\
Maxout+AT & 32.47 & 42.58 & 67.54 & 85.23 (1.14) & 64.98 & 56.03 \\
\rowcolor{TableHighlight}
Tropical+AT & 31.69 & 42.57 & 67.08 & 66.35 (1.04) & 64.67 & 54.84 \\
\bottomrule
\end{tabular}
\end{threeparttable}
\end{adjustbox}
\end{table}

\subsection{Decision Boundaries Visualized}\label{sec:decision_toy_prob}

\begin{example}\label{eg:contour_mnist}
    Let us consider a tropical CNN with the same base model as the MNIST and SVHN models, except the layer before our tropical layer contains only three neurons. Despite this reduction in neurons, a very accurate model can be trained on the 10 classes of the MNIST dataset ($97.5\%$ test accuracy). We then took the trained weights of our tropical layer (10 neurons with three weights each), projected them onto the three dimensional tropical projective torus $\mathbb{R}^3/\mathbb{R}{\bf 1}$, which is isomorphic to the two dimensional Euclidean space $\mathbb{R}^2$, by subtracting all weights by the middle weight, and computed a contour plot that shows tropical distances to the nearest class of the 10 classes. Following this, we fed a subset of training data from each class into the model in order to capture the outputs at our three-neuron layer. We then projected the outputs of the training data at this layer onto the tropical projective torus in the same manner. From this, we can compute the decision boundary of our network as they relate to the weights in our tropical layer. Figure \ref{fig:contour_mnist} and \ref{fig:voronoi_mnist} show a two dimensional representation of our neural network decision boundaries in the tropical projective torus as well as where the training data falls relative to the tropical weights, providing an intuitive visual of the tropical decision boundary described in Section \ref{sec:Boundary}. We provide a ReLU CNN analog to this tropical CNN example in Appendix \ref{appendix:relu_decision}.
\begin{figure}[h!]
\centering
\begin{subfigure}[b]{0.48\columnwidth} % Adjust this value as needed
    \centering
    \includegraphics[width=\linewidth]{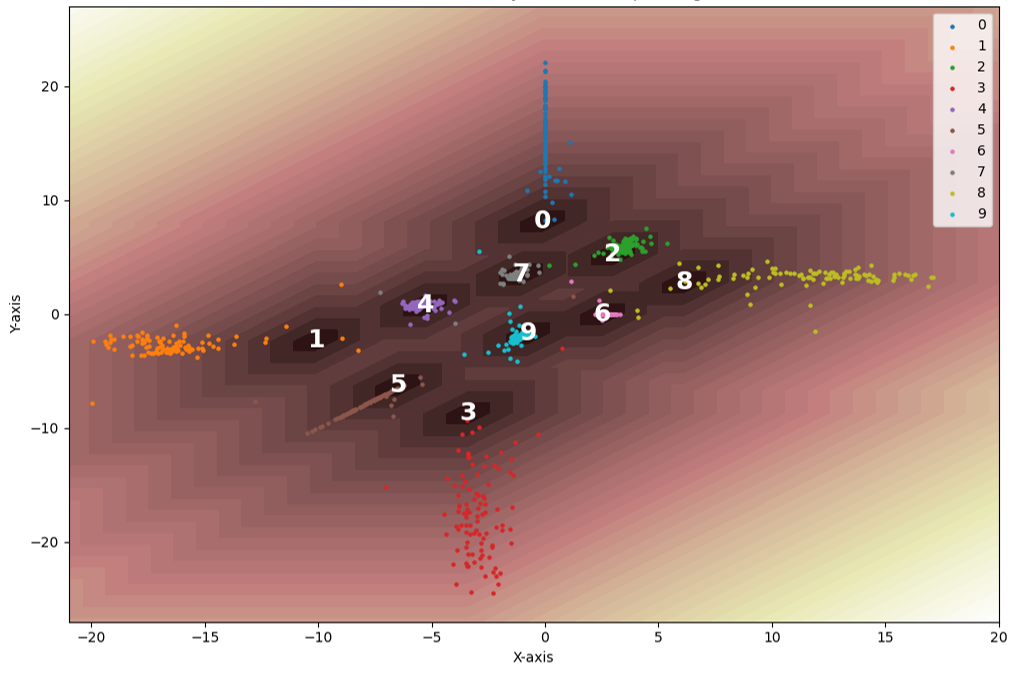}
    \caption{Heat-map plot for distance from trained weights with samples of input data as they are fed into the tropical logit layer.}
    \label{fig:contour_mnist}
\end{subfigure}
\hfill % Adds a little space between the subfigures
\begin{subfigure}[b]{0.48\columnwidth} % Adjust this value as needed
    \centering
    \includegraphics[width=\linewidth]{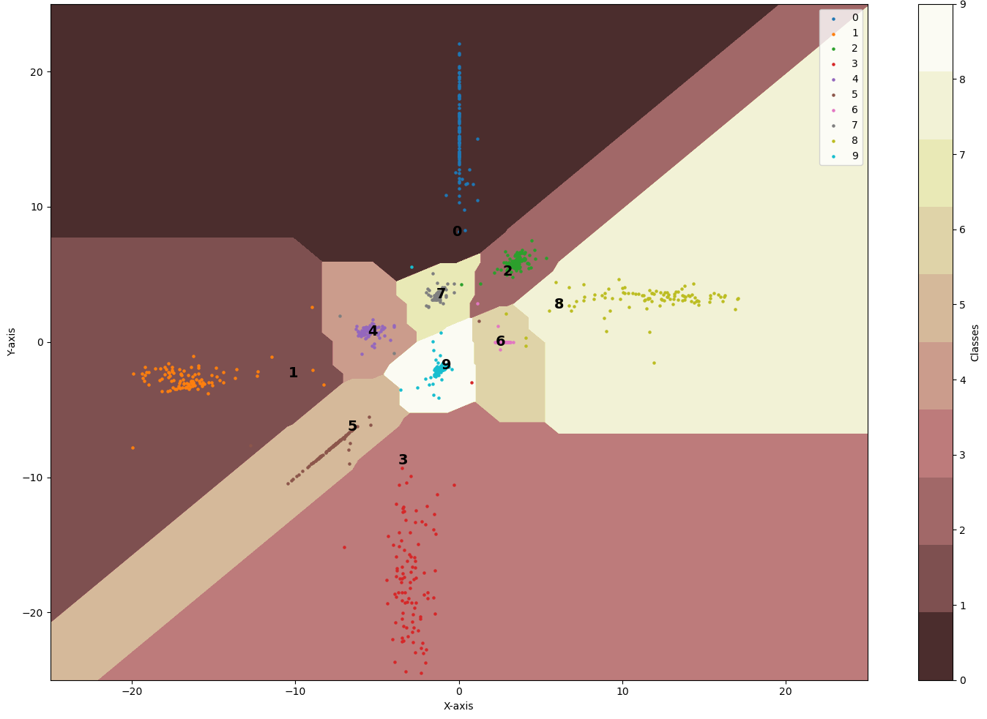}
    \caption{Voronoi cell plot for distance from trained weights with samples of input data as they are fed into the tropical logit layer.}
    \label{fig:voronoi_mnist}
\end{subfigure}
\caption{Decision boundaries MNIST-trained example.}
\label{fig:combined_mnist}
\end{figure}
\end{example}

\subsection{Discussion}
For CIFAR-10 and MNIST datasets, the test error for the tropical model more closely matched the robust models across all attacks, while it maintained parity with standard models on SVHN. However, outside of the class of PGD attacks (ie CW and SPSA), the tropical model still more closely resembled the robust models in test error on SVHN. 

 Across attacks and data sets, the AT effect of straitening of decision boundaries appears to negate the effect of the tropical embedding layer on the decision boundary discussed in this paper~\cite{moosavidezfooli2018robustness}. Notably, the CW attack was able to find more adversarial examples on a tropical model with AT than a tropical model without AT, although at a greater mean $\ell_2$.

The elegance of our tropical CNN is that we are able to achieve robust characteristics and maintain a high level of accuracy, while remaining computationally efficient and without increasing model size. This is especially desirable since state-of-the-art techniques for ensuring robustness in neural networks are prohibitively expensive in terms of computational budget or in terms of accuracy lost on unperturbed data. Our tropical model is able to maintain a parity with benchmark models in terms of accuracy, without noticeably increasing computation time or number of model parameters. Precise training times and parameter counts are details in \ref{appendix:comp_times}. The MMR models, for instance, showed outstanding results in terms of robust characteristics, but took roughly two orders of magnitude longer to train than the tropical model, whereas our tropical model trains in nearly the same time as the ReLU model for every dataset evaluated. The differences in attack results in Table \ref{tab:mnist_attack_results} between the MMR and tropical models is stark as the MMR model far outperformed on all attacks except for the CW and SPSA attack, but the idea that we can achieve robust properties in a neural network with a simple change to how 1 layer connects, without damaging computational efficiency in a material way, is an novel way to apply the properties of tropical geometry in neural networks. 

One can imagine that the computational complexity of our tropical model would increase as the dimensionality of the tropical layer itself increases given we have $max$ and a $min$ operations in our formulation, but embedding the tropical layer as the last layer keeps the dimensionality relatively low and depends only on the number of classes the model is trying to predict and the number of neurons in the layer preceding the final layer. The training times observed in \ref{appendix:comp_times} show that the models in our experiment experienced no adverse training time impact, especially when compared the ReLU counterpart of the exact same construction.

\subsection{Adversarial Attacks on Classifiers}\label{section:attacks}
Here we consider \textit{classifiers} that can be viewed as functions $f: \mathbb{R}^d \times \Theta \rightarrow [C]$, where $d \in \{1, 2, \dots\}$ is the number of input features, $\Theta$ is a (finite-dimensional) space of feasible model parameters (e.g., the weights and biases of a neural network), and $[C] := \{1, \ldots, C\}$ for $C \in \{2, 3, \ldots, \}$ is a finite set of class labels. Given $\theta \in \Theta$, an input with features $x \in \mathbb{R}^d$ is predicted to be of class $f(x, \theta) \in [C]$. The model parameter $\theta$ is typically chosen with the aid of a finite set $D \subset \mathbb{R}^d \times [C]$ of $n = |D|$ training examples. More precisely, let $L(\hat{y}, y)$ denote the loss incurred when the predicted class is $\hat{y} \in [C]$ and the true class is $y \in [C]$. The model parameter $\theta$ is typically ``fitted'' to the training dataset $D$ by solving the following optimization problem:

\begin{align*}
   \text{minimize} &\quad \frac{1}{n} \sum_{(x, y) \in D} L\left(f(x, \theta), y\right) \\
       \text{subject to} &\quad \theta \in \Theta 
\end{align*}
where $n$ is the number of observations in the training dataset $D$.

Classifier neural networks take input pairs $(x,y)$, $x \in \mathbb{R}^d$ and $y \in {K}$, and fit a set of parameters $\Theta$ on a function $f(x;\Theta): (x_1,...,x_d) \xrightarrow{} (x_1,...,x_k)$ to minimize an expected loss function $L(f(x,\Theta), y)$. Then the problem can be formulated as an optimization problem:
    $$min_{\Theta} \ L(f(x;\Theta), y)$$

Fixing $\theta \in \Theta$ (e.g., supposing that the model has been fitted), an \textit{adversarial attack} consists of adding perturbations to certain inputs $x \in \mathbb{R}^d$, with the aim of forcing the fitted model to mis-classify these inputs. For example, if the aim is to force the model to mis-classify the input $x$, whose true class is $y \in [C]$, the perturbation $\delta \in \mathbb{R}^d$ is selected by solving the following optimization problem, where $\Delta$ is the set of feasible perturbations:

\begin{align*}
   \text{maximize} &\quad L\left( f(x + \delta, \theta), y \right) \\
   \text{subject to} &\quad \delta \in \Delta
\end{align*}
The feasible region $\Delta$ is often taken to be an $\epsilon-$ball with respect to a specified norm $\|\cdot\|$ on $\mathbb{R}^d$ (e.g., an $\ell_p$ norm); see e.g., \cite{madry2019deep}.
Alternatively, the adversary seeks to force misclassification, analogous to maximizing the expected value of $L$, without changing the true class of the input. The attackers problem is constrained to some budget $\epsilon$, the maximum perturbation with respect to some $\norm{.}{p}$ which maintains perceptual similarity ~\cite{goodfellow2015explaining}. Then, the attackers problem can be formulated as:
$$\max_{\delta} \ L(f(x + \delta; \Theta), y)  \ s.t. \ \norm{\delta}{p} \leq \epsilon$$

We will consider three well-known approaches for generating adversarial attacks. 
The first, Projected Gradient Descent (PGD), consists of iteratively perturbing a clean input in the direction of the gradient of the mapping $\delta \mapsto L\left(f(x + \delta, \theta), y\right)$; see e.g., \cite{madry2019deep}. Specifically, starting with an initial perturbation $\delta_0$, setting $t = 0$, and for $x \in \mathbb{R}^d$ letting $\Pi_\Delta [x] := \argmin_{\delta \in \Delta} \|x - \delta\|$ denote the projection of $x$ onto $\Delta$, PGD generates a sequence of perturbations $\delta_t$ where
$$
   \delta_{t+1} = \Pi_\Delta \left[ \delta_t + \alpha \nabla L\left(f(x + \delta_t, \theta\right), y) \right], \qquad t = 0, 1, \dots
$$
using a given step size $\alpha$.
$$ x^{t+1} = \Pi_S(x_0) (x_t + \alpha \nabla_x L(f(x,\Theta),y)$$
Where $\Pi_S(x_0)$ is the projection onto the set $S(x_0) = \{x \in \mathbb{R}^d \ | \ \norm{x - x_0}{p} < \epsilon\}$
The second approach, due to Carlini \& Wagner \cite{Carlini17}, is based on solving an optimization problem to find the smallest perturbation (with respect to a given $\ell_p$ norm) that will cause the input $x$ to be mis-classified as being from a target class $\tau$. Specifically, letting $f: \mathbb{R}^d \rightarrow \mathbb{R}$ be a function such that $f(x + \delta) \leq 0$ if and only if $x + \delta$ is classified as being from class $\tau$, and letting $c$ be a positive constant, the optimization problem is to
\begin{align*}
   \text{minimize} &\quad \| \delta \|_p + c \cdot f(x + \delta) \\
   \text{subject to} &\quad x + \delta \ \ \text{is a valid input}
\end{align*}
Here, ``valid input'' can mean, for example, that $x + \delta$ is an image with valid pixel values (e.g., on $[0, 1]$). Carlini \& Wagner \cite{Carlini17} experimented with various objective functions $f$ and constants $c$ in order to develop tailored $\ell_2$, $\ell_0$, and $\ell_\infty$ attacks able to defeat (at the time) state-of-the-art defenses. 
The following example may explain the possible reason why tropical CNN is robust against CW.
\begin{example}
We consider the perturbation $\delta$ of a feature $x \in \mathbb{R}^3/\mathbb{R}{\bf 1}$ by Carlini \& Wagner (CW) attack
in the same setting as in Figure~\ref{fig:bisector_stability} and Lemma~\ref{lemma:bisector_classification}, where the bisector between $w^*_{1, 1}=(0, 0, 0)$ and $w^*_{2, 1} = (1, w, 0)$ is the decision boundary.
Without loss of generality, we assume that $x$ belongs to the class for $w^*_{1, 1}$.
For simplicity, we consider there is no hidden layer (=logistic regression).
Then the perturbation $\delta$ to make $x':=x+\delta$ look like $w^*_{1, 1}$ is given by solving the following optimization:
\[
   \min_{x'} \quad  -d_{\rm tr}(w^*_{2, 1}, x') + d_{\rm tr}(w^*_{1, 1}, x') + c d_{\rm tr}(x, x')
\label{eq:tropicalCW1}.
\]
This is because we share the term $d_{\rm tr}(w^*_{1, 1}, x)$ in the loss function of the tropical CNN and the constraint function of CW.
We define a gradient flow for the distance function as in Figure~\ref{fig:gradient} (left).
(This is basically the gradient of the distance function except at the boundary, on which we selected one natural flow from the subgradient.)
The gradient flow when $w=0.5$ and $c=0$ is illustrated in Figure~\ref{fig:gradient} (right).
The gradient flow is rather inefficient and it can be even $0$ in the green regions.
Note that the small term proportional to $c$ that attracts $x'$ to $x$ should be added further.
Then if $x$ is in one of the green regions, $x'$ just goes back to $x$.
This result may support the idea that tropical CW attack is rather inefficient.
\end{example}
\begin{figure}[h]
   \centering
    \includegraphics[width=1\textwidth]{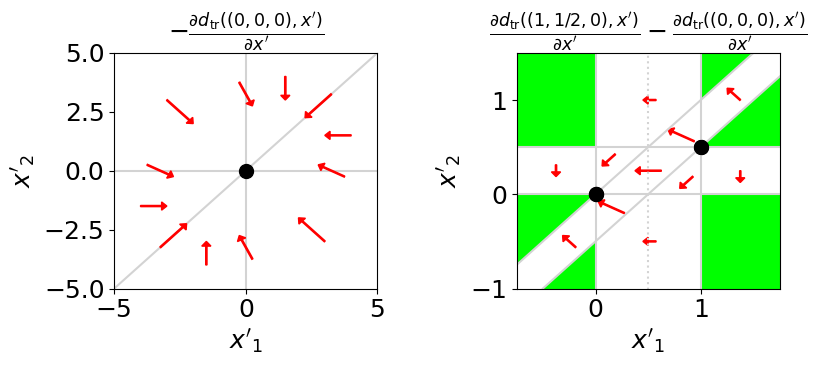}
   \caption{The gradient flow of the distant function from the origin (left) and the gradient flow for the CW attack with $c=0$ (right). The dotted line ($x=0.5$) represents the decision boundary. The gradient flow is zero at the green regions.}
   \label{fig:gradient}
\end{figure}

Carlini and Wagner's (C\&W) attack bypasses the loss function and maximizes the difference between the logits associated with the input class, $Z(x)_i$, and alternative adversary class, $Z(x)_t$ \cite{Carlini17}:
$$\min_{x,t} \ Z(x)_i - Z(x)_t$$ $$s.t. \norm{x - x_0}{p} < \epsilon$$ 
Finally, we consider a ``gradient-free'' method called Simultaneous Perturbation Stochastic Approximation (SPSA) \cite{spall::2003}, which was used by Uesato et al.~\cite{uesato2018adversarial} to generate adversarial attacks.

The final method, collectively known as gradient-free, attacks models defended by obfuscating the loss function gradient  ~\cite{uesato2018adversarial},~\cite{athalye2018obfuscated}. There are a variety of gradient-free attacks with different approaches, but in general, they are less powerful than the gradient based approaches listed previously. In this paper, we use the SPSA attack ~\cite{uesato2018adversarial} which takes stochastic sample perturbations within an $\epsilon > 0$ ball of $x_0$ and estimates loss gradient based off their difference.

Possible reason why tropical CNN is robust against CW:
We assume that $c \in [C]$ is the correct class for an observation $x_c \in \mathbb{R}^d/\mathbb{R}{\bf 1}$ in the dataset such that the true response for an observation $x_c$ is $c \in [C]$ and $\hat{w}^{(c)} \in \mathbb{R}^d/\mathbb{R}{\bf 1}$ is the estimated weight for the class $c\in [C]$. Suppose $\hat{w}^{(c')} \in \mathbb{R}^d/\mathbb{R}{\bf 1}$ is the estimated weight for the class $c'\in [C]$ such that $c' \not = c$. We would like to note that a tropical geodesic between two points is not unique (for example, Example 5.21 in \cite{ETC}).  In fact the set of all possible tropical geodesics between two points forms a tropical polytrope, which is a classical polytope and tropical polytope in the tropical projective torus.  Let $P(\hat{w}^{(c)}, \hat{w}^{(c')})$ be the set of all tropical geodesics between $\hat{w}^{(c)}$ and $\hat{w}^{(c)}$, which is a polytrope containing $\hat{w}^{(c)}$ and $\hat{w}^{(c')}$.  Also note that the decision boundary $\mathcal{B}(\hat{w}^{(c)}, \hat{w}^{(c')})$ defined by $\hat{w}^{(c)}$ and $\hat{w}^{(c')}$ dissects $P(\hat{w}^{(c)}, \hat{w}^{(c')})$ such that
\[
r_{c, c'}:= d_{\rm tr}(\hat{w}^{(c)}, \mathcal{B}(\hat{w}^{(c)}, \hat{w}^{(c')})) = d_{\rm tr}(\hat{w}^{(c')}, \mathcal{B}(\hat{w}^{(c)}, \hat{w}^{(c')}))
\]
and 
\[
\mathcal{B}(\hat{w}^{(c)} , \hat{w}^{(c')}) \cap P(\hat{w}^{(c)} , \hat{w}^{(c')})  \not = \emptyset .
\]
We consider a case which the perturbation by  Carlini \& Wagner (CW) attack $x + \delta \in \mathbb{R}^d/\mathbb{R}{\bf 1}$ is near the decision boundary $\mathcal{B}(\hat{w}^{(c)}, \hat{w}^{(c')})$ defined by $\hat{w}^{(c)}$ and $\hat{w}^{(c')}$.  This means that $x + \delta \in P(\hat{w}^{(c)} , \hat{w}^{(c')})$.  We also assume that $x \in P$ without loss of generality since if $x \not \in P(\hat{w}^{(c)} , \hat{w}^{(c')})$ then for some iterations $t$, the perturbation by CW attack $x:= x + \delta_t $ is closer to the decision boundary, i.e., $x:= x + \delta_t  \in P$.

Without loss of generality, we assume that $\hat{w}^{(c)} = 0$.  
Here we have 
\begin{eqnarray}\label{eq:tropicalCW}
   \text{minimize in }\delta &\quad  d_{\rm tr}(0, x + \delta) \\\nonumber
   \text{subject to} &\quad d_{\rm tr}(0, x + \delta) \geq r_{c, c'}.
\end{eqnarray}

Let 
\[
y = d_{\rm tr}(0, x + \delta).
\]

\begin{definition}[sectors for hyperplane $H_\omega$]
The $k$-th open sector for the max-tropical hyperplane $H^{\max}_\omega$ is
$G^{\max}_\omega(k) = \{x | x_k + \omega_k > x_i + \omega_i \textrm{ for } 1 \leq i(\neq k) \leq d \}$. The $l$-th open sector for the min-tropical hyperplane $H^{\min}_\omega$ is $G^{\min}_\omega(l) = \{x | x_l + \omega_l < x_i + \omega_i \textrm{ for } 1 \leq i(\neq l) \leq d \}$. The $k$-th closed sector for the max-tropical hyperplane $H^{\max}_\omega$ is $\overline{G^{\max}_\omega}(k) = \{x | x_k + \omega_k \geq x_i + \omega_i \textrm{ for } 1 \leq i(\neq k) \leq d\}$. The $l$-th closed sector for the min-tropical hyperplane $H^{\min}_\omega$ is $\overline{G^{\min}_\omega}(l) = \{x | x_l + \omega_l \leq x_i + \omega_i \textrm{ for } 1 \leq i(\neq l) \leq d \}$.
\end{definition}

\begin{lemma}[Lemma 17 in FW point paper]\label{lm:gradient}
For $x \in G^{\max}_\omega(k) \cap G^{\min}_\omega(l)$, $d_{\rm tr}(0, x+\omega) = x_k + \omega_k - x_l - \omega_l$ and its gradient is given as
$\frac{\partial d_{\rm tr}(0, x+\omega)}{\partial \omega_i} 
= \delta_{ik} - \delta_{il}$, where $\delta_{ij}$ is the Kronecker's delta such that
\[
\delta_{ij} = \begin{cases}
   1 & \mbox{if } i = j\\
   0 & \mbox{otherwise.}
\end{cases}
\]
\end{lemma}
\begin{lemma}[Lemma 29 in Barnhill and Sabol et al.]\label{lm:lower}
The gradient given in Lemma \ref{lm:gradient} changes when and only when $w$ cross the sector boundary defined by the tropical min and max hyperplanes of $x$.
\end{lemma}
\begin{remark}
   Since the problem in \eqref{eq:tropicalCW} is a special case of the problem in \cite{BSYM}.  By Theorem 2 in \cite{BSYM}, the gradient in Lemma \ref{lm:gradient} achieves its minimum. Barnhill and Sabol et al.~discussed that when we have a lower dimensional set of optimal solutions pass through the set by Lemma \ref{lm:lower}.  We observed this behaviors in our experiments.
\end{remark}

\subsection{Defenses Against Adversarial Attacks}\label{section:AT}
The adversarial training (AT) defense, initially proposed in~\cite{goodfellow2015explaining}, consists of training a model on both clean and perturbed inputs. Madry et al.~\cite{madry2019deep} used Projected Gradient Descent (PGD) on the loss function to generate a 'best' first-order attack, and continued to train the model until it classified the adversarial inputs correctly. Incorporating these attacks in training effectively results in a more linear decision boundary in the neighborhood of natural inputs~\cite{moosavidezfooli2018robustness}. By reducing the curvature of the decision boundary, the classifier becomes less vulnerable to 'shortcut' attacks that use the loss gradient to perturb across the closest decision boundary.

Regularization-based defenses take a more direct approach to achieve a similar effect. Often, they penalize small distances between observations and decision boundaries, effectively pushing the closest segments away and straightening the loss function. This is achieved by incorporating a $\ell_p$ $\epsilon$-ball into the defenders problem either by minimizing the loss within an $\epsilon$-ball of the input as in ~\cite{wong2018provable}, or by directly widening the linear region around training inputs with respect to some $\ell_p$ norm~\cite{croce2019provable},~\cite{croce2020provable}.

With some exceptions ~\cite{croce2020provable}, robust models are robust only to the specific norm on which they were trained ~\cite{tramèr2019adversarial}, and require significant computational resources to train, either through iterative training as in AT or through calculation of the regularization term. In this paper we show simultaneous robustness against attacks based on $\ell_1$, $\ell_2$, and $\ell_{\infty}$ norms, using less computational time than adversarial training.

\section{Conclusion}
Motivated by the low-rank nature of image classification, we apply tools from  tropical geometry with the max-plus algebra over the tropical semiring to CNNs, and introduce a {\em tropical CNN}. We also demonstrate that this novel CNN is robust against white box adversarial attacks via computational experiments on image datasets.  We show that it is especially robust against CW attacks.    

In this paper we focus on image datasets.  Deep neural networks have also been applied to text analysis (for example, ~\cite{Suissa2023}).  For text, researchers apply tools from linguistics that were originally developed for genetics and genomics.  Yoshida et al.~in \cite{tropicalNN} applied neural networks with tropical embedded layer to phylogenomics.  Therefore it would be interesting to apply tropical CNNs to datasets for text analysis and mining.

Theorem \ref{upperbound} assumes Gaussian data about $w_c$. However, \ref{fig:contour_mnist} suggests that the convolutional and ReLU layers prior to the embedding layer generate features that do not have a symmetric distribution about $w_c$, indicating the possibility of a tighter upper bound on $P(d_{\rm tr}(w_c, X) \geq r)$.

From the computational results in Section \ref{subsec:result} on tropical CNNs with adversarial training, we observe that adversarial training makes the robustness of tropical CNNs similar to  the  robustness of CNNs with ReLU activators with  adversarial training (i.e., adversarial training makes tropical CNNs less robust against adversarial attacks).  These results suggest that the adversarial training effect \cite{moosavidezfooli2018robustness} may be negating the effect of the tropical embedding layer.  It is interesting to investigate the geometry of the effects of adversarial training to tropical CNNs.

\section*{Acknowledgement}
RY is partially supported by the National Science Foundation (Grant No. DMS-1929348).
KM is partially supported by JSPS KAKENHI Grant Numbers JP22K19816 and JP22H02364.

\bibliographystyle{unsrtnat} 
\bibliography{refs}
\appendix

\section{Additional Computational Results}\label{appendix:experiment}

Here we capture the details of our implementation of a tropical CNN experiment that was briefly details in Section \ref{sec:comp_experiments}. %Further details and the code used can also be found in the GitHub repository for the project: https://github.com/KurtPask/TropicalNN. 

\subsection{Model Construction}\label{appendix:model_construction}

Table \ref{table:base-model} shows the neural network construction of our base model used to train models on the MNIST and SVHN datasets. 

{\begin{table}[h]
\centering
\caption{Base Model for MNIST and SVHN experiments.}
\footnotesize
\begin{tabular}{c|c|c}
\textbf{Layer} & \textbf{Activation} & \textbf{Key Parameters} \\
\hline
Convolution & ReLU & 64, 3x3 windows, 1x1 strides \\
Max Pooling & & 2x2 windows, non-intersecting \\
Convolution & ReLU & 64, 3x3 windows, 1x1 strides \\
Max Pooling & & 2x2 windows, non-intersecting \\
Convolution & ReLU & 64, 3x3 windows, 1x1 strides \\
Flatten & & Flatten feature map \\
Fully connected & ReLU & 64 neurons \\
\end{tabular}
\label{table:base-model}
\end{table}

The models used in our experiment we refer to as Tropical, ReLU, Maxout, and MMR. The Tropical model is defined as adding 1 tropical embedding layer to the base model as defined in Equation \ref{eq:embed}, which produces the logits used for classification instead of a typical, fully-connected layer. Our logit logic differs from a typical full-connected layer because the minimum value is considered the correct class and not the maximum value. The value's in our case are the tropical distance from the output of the base model to the weights of final layer, and thus we want to minimize this distance. To account for this in our experiment we take $z=50-d_{\rm tr}(x,w)$, where $z$ is our logits, $x$ the output from the base model, and $w$ the trained weights of the final layer. By taking the negative of the distance values ($d_{\rm tr}(x,w)$), this allows us to keep the logic that the maximum value is correct class and adding an arbitrary amount\footnote{50 in our case because $d_{\rm tr}(x,w)$ was typically observed between 5 and 30. In practice, one could take the negative of the output without adding this arbitrary amount.} only served to prevent any logic issues among our attack algorithms in dealing with negative numbers. For the ReLU model we connect a typical, fully-connected layer to the base model in order to produce logits for classification. For the Maxout models, we connect the layers outlined in Table \ref{table:maxout-model} to the base model. This implements the maxout units as defined in~\cite{goodfellow2013maxout}, which appears to show similar construction as our tropical model given the logits are produced by subtracting 2 values, but we will show through our experiment results that our model and the maxout model behave very differently. The MMR model is identical to the ReLU model, but uses the MMR-Universal regularizer defined in~\cite{croce2020provable} and adapted to our model. The MMR-Universal regularizer attempts "enforce robustness wrt $\ell_1$- and $\ell_\infty$-perturbations and show how that leads to the first provably robust models wrt any $\ell_p$-norm for $p\ge1$"~\cite{croce2020provable}. For the Tropical, ReLU, and Maxout models, we also train and evaluate models that are trained using adversarial examples as defined in Section \ref{section:AT}.

\begin{table}[ht]
\centering
\caption{Maxout Model}
\footnotesize
\begin{tabular}{c|c|c}
\textbf{Layer} & \textbf{Activation} & \textbf{Key Parameters} \\
\hline
Fully connected 1 & ReLU & 10*100 neurons,connected to base \\
Fully connected 2 & ReLU & 10*100 neurons,connected to base \\
Dropout & & 0.5 dropout rate \\
Maxout 1 & Maxout & 10 units, max from FC1 \\
Maxout 2 & Maxout & 10 units, max from FC2 \\
Final Layer & Maxout 1 - Maxout 2 &  \\
\end{tabular}
\label{table:maxout-model}
\end{table}

For the MNIST and SVHN dataset, we compared the tropical model's performance against attacks to the ReLU, Maxout, and MMR models. However, for CIFAR-10 we just compare to the ReLU and Maxout models, as re-creating the MMR-Universal regularizer for a model as large as ResNet50 exceeded our computational budget for the research. 

\subsection{Tropical Layer in Tensorflow}

We built our models in Tensorflow utilizing the functional API framework. To build a tropical CNN using this framework, we need to create a Layer class. Below is the python code used to build the layer:

\begin{lstlisting}[style=pythonstyle, caption={Tropical Layer Class in Tensorflow}, label=pythoncode]
class TropEmbedMaxMin(Layer):
    '''
    Custom TensorFlow layer implementing Tropical Embedding for max-min distances.
    '''

    def __init__(self, units=2, initializer_w=initializers.random_normal, lam=0.0, axis_for_reduction=2, **kwargs):
        '''
        Initializes the TropEmbedMaxMin layer.

        Parameters
        ----------
        units : int, optional
            Number of output units (default is 2).
        initializer_w : initializer function, optional
            Weight initializer function (default is random_normal).
        lam : float, optional
            Regularization parameter (default is 0.0).
        axis_for_reduction : int, optional
            Axis for reduction in distance calculation (default is 2).
        **kwargs : dict
            Additional keyword arguments.
        '''
        super(TropEmbedMaxMin, self).__init__(**kwargs)
        self.units = units
        self.initializer_w = initializer_w
        self.lam = lam
        self.axis_for_reduction = axis_for_reduction

    def build(self, input_shape):
        input_dim = input_shape[-1]  # Extract the last dimension from input_shape
        self.w = self.add_weight(name='tropical_fw',
                                 shape=(self.units, input_dim),
                                 initializer=self.initializer_w,
                                 regularizer=TropRegIncreaseDistance(lam=self.lam),
                                 trainable=True)
        self.bias = self.add_weight(name='bias',
                                    shape=(self.units,),
                                    initializer="zeros",
                                    trainable=True)
        super(TropEmbedMaxMin, self).build(input_shape)

    def call(self, x):
        '''
        Performs the forward pass of the TropEmbedMaxMin layer.

        Parameters
        ----------
        x : tensorflow tensor object
            Input tensor.

        Returns
        -------
        trop_distance : tensorflow tensor object
            Output tensor after applying Tropical Embedding for max-min distances.
        '''
        x_reshaped = reshape(x, [-1, 1, self.w.shape[-1]])  # Reshape input data
        x_for_broadcast = repeat_elements(x_reshaped, self.units, 1)  # Repeat input for broadcasting
        result_addition = x_for_broadcast + self.w  # Calculate addition of input and weights
        trop_distance = reduce_max(result_addition, axis=(self.axis_for_reduction)) - reduce_min(result_addition, axis=(self.axis_for_reduction)) + self.bias  # Calculate tropical distances with bias
        return trop_distance


\end{lstlisting}

\subsection{Attack Hyperparameters}

All attacks used, except for CW, require a norm constraint hyperparameter\footnote{The specified values of all $\epsilon$ values outlined here are chosen considering that the images in these datasets are normalized such that pixel values lie within the range [0,1].}, $\epsilon$. The CW algorithm attempts to find adversarial examples while minimizing the $\ell_2$ perturbation, but does not constrain the perturbation in the $\ell_2$-ball, so we do not define an $\epsilon$ for it. Given this, for all our $\ell_\infty$, we use common $\epsilon's$ given the dataset. For MNIST, we use $\epsilon_\infty=0.1$, for SVHN and CIFAR-10 we use $\epsilon_\infty=\frac{4}{255}$. For all datasets we compute the $\ell_2$ constraint ($\epsilon_2$) as $\epsilon_2=\sqrt{\epsilon_\infty^2 * n}$ where $n$ is the number of input elements/pixels. For MNIST, $\epsilon_2=2.8$, and for SVHN and CIFAR-10, $\epsilon_2=0.87$. For all datasets we compute the $\ell_1$ constraint ($\epsilon_1$) as $\epsilon_1=\epsilon_2*2$. For MNIST, $\epsilon_1=5.6$, and for SVHN and CIFAR-10, $\epsilon_1=1.74$. The $\epsilon_2$ and $\epsilon_1$ constraints provided were chosen because they achieved similar attack performance to the $\ell_\infty$ attacks.

Beyond the $\epsilon$ constraint choices, we outline the other key hyperparameters for each attack used in the experiment:

\begin{itemize}[itemsep=0pt,parsep=2pt,left=0pt]
\small
\item Sparse $\ell_1$ Descent (SLIDE):
\begin{itemize}[itemsep=0pt,parsep=0pt,left=0pt,label=-]
    \item steps = 100
    \item $\epsilon$ = {MNIST: 5.6, SVHN/CIFAR: 1.74}
    \item step size = 0.01
    \item percentile = 99
\end{itemize}

\item $\ell_2$ and $\ell_\infty$ Projected Gradient Descent (PGD):
\begin{itemize}[itemsep=0pt,parsep=0pt,left=0pt,label=-]
    \item steps = 100
    \item $\ell_2$ $\epsilon$ = {MNIST: 2.8, SVHN/CIFAR: 0.87}
    \item $\ell_\infty$ $\epsilon$ = {MNIST: 0.1, SVHN/CIFAR: $\frac{4}{255}$}
    \item step size = 0.01
    \item random start within epsilon ball = True
\end{itemize}

\item Carlini and Wagner (CW):
\begin{itemize}[itemsep=0pt,parsep=0pt,left=0pt,label=-]
    \item abort early = True
    \item max iterations per binary search step = 1000
    \item number of binary search steps = 10
    \item confidence = 0
    \item initial constant = 10
    \item learning rate = 0.1
\end{itemize}

\item Simultaneous Perturbation Stochastic Approximation (SPSA) 
\begin{itemize}[itemsep=0pt,parsep=0pt,left=0pt,label=-]
    \item $\epsilon$ = {MNIST: 0.1, SVHN/CIFAR: $\frac{4}{255}$}
    \item number of iterations = 100
    \item learning rate = 0.01
    \item delta = 0.01
    \item spsa samples = 128
    \item spsa iters = 1
    \item early stop loss threshold = 0.0
\end{itemize}
\end{itemize}

\subsection{Model Complexity and Computation Times}\label{appendix:comp_times}

We defined our model structure in Appendix Section \ref{appendix:model_construction}, including the base models used as well as each model's structure connected to the base model. To get a sense of the size of each model in terms of trainable parameter, Table \ref{tab:trainable_parameters} shows the raw count of trainable parameters as reported using the Tensorflow "summary()" method.

\begin{table}[htbp]
\centering
\caption{Trainable Parameters}
\begin{tabular}{lccc}
\toprule
 & MNIST & SVHN & CIFAR \\
\midrule
ReLU & 112,074 & 141,898 & 23,666,378 \\
Maxout & 241,424 & 271,248 & 23,795,728 \\
Tropical & 112,074 & 141,898 & 23,666,378 \\
MMR & 112,074 & 141,898 & - \\
\bottomrule
\end{tabular}\label{tab:trainable_parameters}
\end{table}

The experiment computation (training models and attacking models) was completed using high performance computing (HPC) resources. Each model except for the MMR SVHN model were trained for 100 epochs. In order to get a sense of the computational burden imposed by adding our tropical layer, we captured computational times across the runs. Please note that the compute resources utilized were not optimized and not the exact same across each model. Much of the resource allocation was done by trial-and-error. That said, we will try to give the best apples-to-apples comparison below of computation times, given the models that used the same resources.
Table \ref{tab:normal_training} shows the training times in seconds for the models that did not employ any adversarial or robust training techniques. Each model was scheduled to run on the HPC cluster with 4 GPU's and 10 CPU's, but the SVHN model trained with 0 GPU's and 10 CPU's. Each column used the same resources.

\begin{table}[htbp]
\centering
\caption{Computation time in seconds. Each column used the same GPU/CPU compute resources.}
\begin{tabular}{lccc}
\toprule
 & MNIST & SVHN & CIFAR \\
\midrule
ReLU & 253 sec & 4,318 sec & 2,445 sec \\
Maxout & 265 & 4,467 & 2,412 \\
Tropical & 268 & 4,267 & 2,409 \\
\bottomrule
\end{tabular}\label{tab:normal_training}
\end{table}

Table \ref{tab:at_training} shows the training time in seconds for the models trained using he adversarial training method of training on perturbed examples using $\ell_\infty$ PGD. Each model was trained with 8 GPU's and 30 CPU's. 

\begin{table}[htbp]
\centering
\caption{Computation time in seconds. Each column used the same GPU/CPU compute resources.}
\begin{tabular}{lccc}
\toprule
 & MNIST & SVHN & CIFAR \\
\midrule
ReLU+AT & 5,323 sec & 6,738 sec & 54,855 sec \\
Maxout+AT & 6,012 & 7,511 & 54,978 \\
Tropical+AT & 5,870 & 7,214 & 54,791 \\
\bottomrule
\end{tabular}\label{tab:at_training}
\end{table}

Table \ref{tab:mmr_training} shows the training times for the 2 MMR models built. Recall that the ResNet-50 base model used for the other CIFAR-10 models was infeasible for the project computational resources we had and thus was not constructed. The training of both models occurred with 0 GPU's and 30 CPU's and had to be done in batches of 32 examples. The research team struggled to find the correct HPC configuration to train the model with GPU's as the MMR model requires high dimension matrix computations and the memory and batch size configuration could not be found to be able to run on the GPU's available to us, thus 30 CPU's ended up being optimal in terms of training time on strictly CPU resources.

\begin{table}[htbp]
\centering
\caption{Computation time in seconds. 30 CPU cores were used to train the MMR model.}
\begin{tabular}{lccc}
\toprule
 & MNIST & SVHN (only 15 epochs completed) & CIFAR \\
\midrule
MMR & 97,380 sec & 261,900 sec & - \\
\bottomrule
\end{tabular}\label{tab:mmr_training}
\end{table}

The tables should show ample evidence that our model, given the same resources, achieves nearly computational time parity with its ReLU and maxout counterpart and is notably less expensive than powerful techniques such as employing the MMR-Universal regularizer. Due to the operations that take place (max and min) within our tropical layer, one can expect computation time to be higher if the problem were higher dimension. Particularly, in our case we were building models to classify 10 classes. Should the class number be higher, or the layer preceding the logit layer be higher, this might strain the computational time parity observed in our experiment. 

\section{Decision Boundaries of ReLU Neural Networks}\label{appendix:relu_decision}

Expanding on the decision boundary toy problem articulated in Example \ref{eg:contour_mnist}, we can build an analogous visual of the Voronoi cells that define the boundary for a ReLU activated model of the exact same construction as our tropical model above, the only difference being we use a normal, affine fully-connected layer to produce our logits. Because the projection from $\mathbb{R}^3$ to $\mathbb{R}^2$ is not applicable in the ReLU model, we must visualize in $\mathbb{R}^2$. Figure \ref{fig:relu_voronoi_cells} shows three ``slices'' of $\mathbb{R}^3$. 

\begin{figure}[h!]
    \centering
    \includegraphics[width=0.9\textwidth]{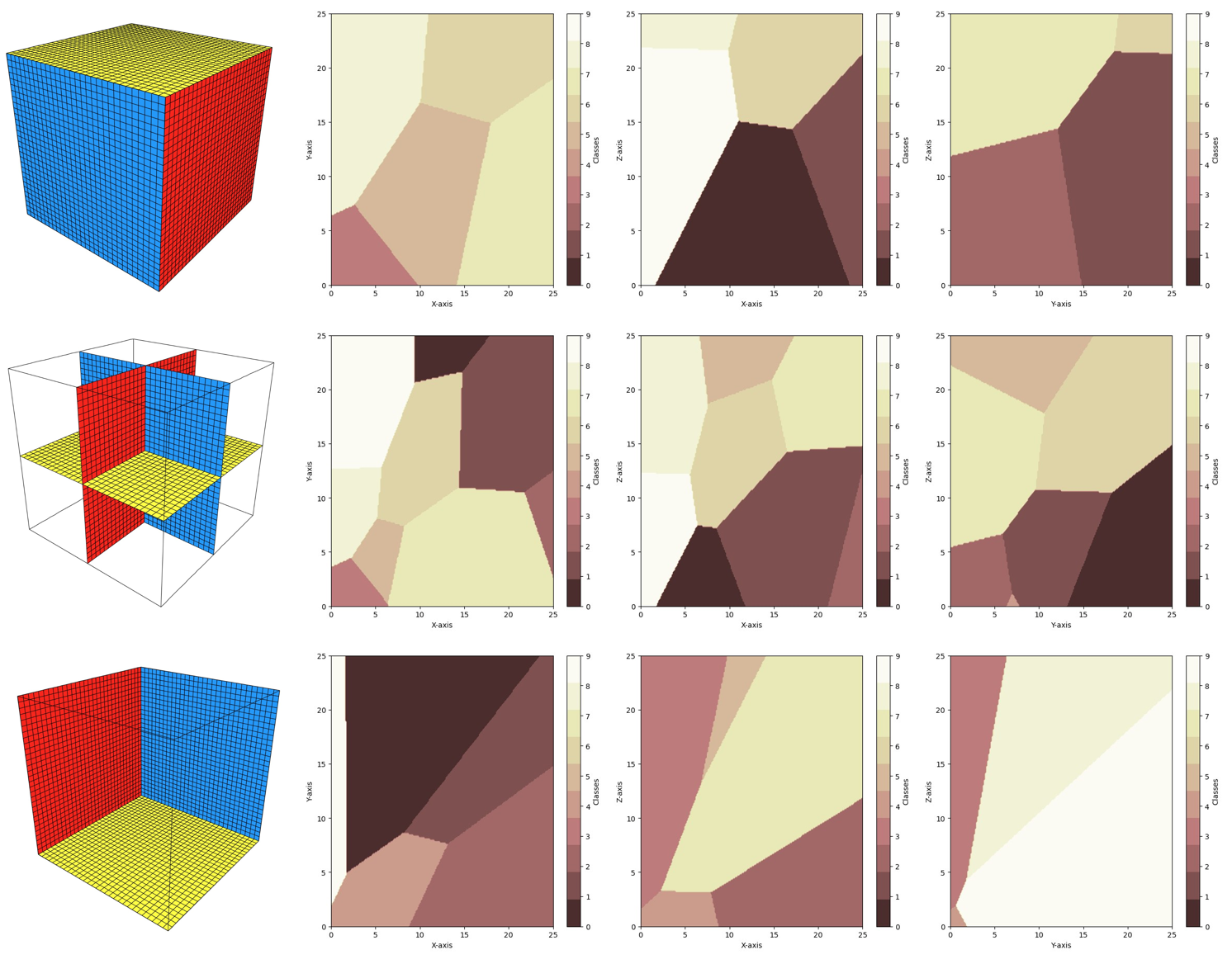}
    \caption{Voronoi diagram for MNIST-trained example ReLU model.}
    \label{fig:relu_voronoi_cells}
\end{figure}

\end{document}